\pdfoutput=1
\documentclass[11pt]{article}

\usepackage{mathtools}
\usepackage{times,bm}
\usepackage{soul}

\usepackage{graphicx}
\usepackage{amsmath}
\usepackage{amsthm}
\usepackage{amssymb}
\usepackage{booktabs}
\usepackage{array}
\usepackage{float}
\usepackage[ruled]{algorithm2e}
\usepackage[noend]{algorithmic}
\usepackage{adjustbox}
\usepackage{wrapfig}
\usepackage{paralist}
\usepackage{enumitem}
\usepackage{tikz}
\usetikzlibrary{arrows,arrows.meta}
\usetikzlibrary{decorations.markings}
\usetikzlibrary{decorations.pathmorphing}
\usetikzlibrary{positioning,shapes,backgrounds}
\usetikzlibrary{calc,intersections,patterns,fit,automata,snakes}
\usetikzlibrary{svg.path}

\usepackage{natbib}
\definecolor{skipcolor}{rgb}{1, .92, .92}
\definecolor{rejoincolor}{rgb}{.9, 1, 0.9}
\definecolor{frameskip}{rgb}{0.95, 0, 0}
\definecolor{framerejoin}{rgb}{0, .75, 0}

\definecolor{gold}{rgb}{1,0.553,0}
\definecolor{lightbrown}{rgb}{0.9305,0.86275,0.8}
\definecolor{fillcopper}{rgb}{0.722,0.451,0.2}
\definecolor{fillsilver}{rgb}{0.753,0.753,0.753}
\definecolor{fillgray}{rgb}{0.4,0.4,0.4}
\definecolor{filllightgray}{rgb}{0.6,0.6,0.6}
\definecolor{fillLightgray}{rgb}{0.7,0.7,0.7}
\definecolor{fillLLightgray}{rgb}{0.8,0.8,0.8}
\definecolor{fillLLLightgray}{rgb}{0.9,0.9,0.9}
\definecolor{fillred}{rgb}{1,0.15,0.15}
\definecolor{filllightred}{rgb}{1,0.3,0.3}
\definecolor{fillblue}{rgb}{0.2,0.35,1}
\definecolor{filllightblue}{rgb}{0.75,0.85,1}
\definecolor{fillgreen}{rgb}{0.2,0.7,0.2}
\definecolor{filllightgreen}{rgb}{0.65,0.95,0.65}
\definecolor{fillgreenbox}{rgb}{0.65,0.95,0.65}
\definecolor{fillpurple}{rgb}{0.7,0.4,0.74}
\definecolor{redRPI}{rgb}{0.839,0,0.1098}%RPI red

\def\math#1{$#1$}

\def\mand#1{$$#1$$}

\def\mymathskip{5.25pt}

\def\mandc#1{\mand{\abovedisplayskip=\mymathskip plus 1pt minus 1pt%
\abovedisplayshortskip=0pt plus 1pt minus 1pt%
\belowdisplayskip=\mymathskip plus 1pt minus 1pt%
\belowdisplayshortskip=0pt plus 1pt minus 1pt%
#1}}

\def\mldc#1{\mld{\abovedisplayskip=\mymathskip plus 1pt minus 1pt%
\abovedisplayshortskip=0pt plus 1pt minus 1pt%
\belowdisplayskip=\mymathskip plus 1pt minus 1pt%
\belowdisplayshortskip=0pt plus 1pt minus 1pt%
#1}}

\def\mld#1{\begin{equation}
#1
\end{equation}}

\def\eqan#1{\begin{eqnarray*}
#1
\end{eqnarray*}}

\def\eqar#1{\begin{eqnarray}
#1
\end{eqnarray}}

\def\frac#1#2{{#1\over #2}}

\DeclareSymbolFont{AMSb}{U}{msb}{m}{n}
\DeclareMathSymbol{\Exp}{\mathord}{AMSb}{"45}
\DeclareMathSymbol{\Prob}{\mathord}{AMSb}{"50}

\def\cl#1{{\cal #1}}

\newcommand{\given}{\mid}

\def\floor#1{
\mathchoice
{\left\lfloor\,#1\,\right\rfloor}
{\big\lfloor\,#1\,\big\rfloor}
{\lfloor\,#1\,\rfloor}
{\lfloor\,#1\,\rfloor}
}

\def\r#1{{\eqref{#1}}}

\newcounter{rmnum}
\def\RN#1{\setcounter{rmnum}{#1}\uppercase\expandafter{\romannumeral\value{rmnum}}}
\def\rn#1{\setcounter{rmnum}{#1}\expandafter{\romannumeral\value{rmnum}}}

\definecolor{shadecolor}{gray}{.85}%
\definecolor{tintedcolor}{gray}{.8}%

{\makeatletter\gdef\reallynopagebreak{\par\nopagebreak\@nobreaktrue}}

\providecommand\remove[1]{}

\newcommand{\vbar}{\ensuremath{\bar v}}

\newcommand{\effect}{\ensuremath{\textsc{eff}_1}}
\newcommand{\sideeffect}{\ensuremath{\textsc{eff}_0}}
\newcommand{\ite}{\textsc{ite}}
\newcommand{\itehat}{\widehat{\textsc{ite}}}

\newcommand{\att}{\textsc{att}}

\providecommand{\remove}[1]{}

\newcolumntype{C}{>{\centering\arraybackslash}m{\wid}}
\newcommand{\bfit}[1]{\textbf{\textit{#1}}}
\newcommand{\wid}{0.3\columnwidth}

\newcommand{\addpic}[1]{\includegraphics[width=\wid]{#1}}

\newcommand{\pKK}{./}

\DeclareMathSymbol{\Prob}{\mathbin}{AMSb}{"50}
\DeclareMathSymbol{\Exp}{\mathbin}{AMSb}{"45}

\newtheorem{theorem}{\bf Theorem}[section]
\newtheorem{lemma}[theorem]{Lemma}

\title{Consistent Causal Inference of Group Effects in Non-Targeted Trials with
  Finitely Many Effect Levels}
\remove{
\clearauthor{%
 \Name{Georgios Mavroudeas} \Email{mavrog2@rpi.edu}\\
 \addr Department of Computer Science, Rensselaer Polytechnic Institute%
 \AND
 \Name{Malik Magdon-Ismail} \Email{magdon@cs.rpi.edu}\\
 \addr Department of Computer Science, Rensselaer Polytechnic Institute%
 \AND
 \Name{Kristin P. Bennett} \Email{bennek@rpi.edu}\\
 \addr Department of Mathematical Sciences \& Department of Computer Science,
  Rensselaer Polytechnic Institute%
 \AND
 \Name{Jason Kuruzovich} \Email{kuruzj@rpi.edu}\\
 \addr Lally School of Management, Rensselaer Polytechnic Institute%
}
}

\author{Georgios Mavroudeas\\
 mavrog2@rpi.edu\\
Computer Science\\
Rensselaer Polytechnic Institute\\
110 8th Street, Troy, NY 12180, USA
  \and 
Malik Magdon-Ismail\\
 magdon@cs.rpi.edu\\
Computer Science Department\\
Rensselaer Polytechnic Institute\\
110 8th Street, Troy, NY 12180, USA
  \and 
Kristin P. Bennett\\
bennek@rpi.edu\\
Mathematical Sciences \& Computer Science\\
Rensselaer Polytechnic Institute\\
110 8th Street, Troy, NY 12180, USA
  \and 
Jason Kuruzovich\\
kuruzj@rpi.edu\\
Lally School of Management\\
Rensselaer Polytechnic Institute\\
110 8th Street, Troy, NY 12180, USA
}

\begin{document}

\maketitle

\begin{abstract}
  \noindent
  A treatment may be appropriate for some group (the ``sick" group)
  on whom it has a positive effect, but it can also have a detrimental effect
  on subjects from another group (the ``healthy" group).
  In a non-targeted trial both sick and healthy subjects may be treated,
  producing heterogeneous effects within the treated group.
  Inferring the correct treatment effect on the sick population is then 
  difficult, because the effects on the different groups get tangled.
  We propose an efficient nonparametric approach to estimating the
  group effects, called {\bf PCM} (pre-cluster and merge).
  We prove its asymptotic consistency in a general setting and
  show, on synthetic data, 
  more than a 10x improvement in accuracy over existing state-of-the-art.
  Our approach applies more generally to consistent
  estimation of functions with a finite range.
\end{abstract}

\section{Introduction}
\label{sec:intro}

%\jason{I wrote up a bit more contextual framing. Happy to fill out the citations if we want to go in this directions.}

%Machine learning is powerful for efficient 
%low cost
%causal effect estimation
%\citep{pearl2019seven}. 
A standard approach to causal effect estimation is 
the targeted 
randomized controlled 
trial (RCT), see
\cite{greenland1990randomization,liberati1986quality,pearl2019seven,rosenbaum2007interference,wager2018estimation}.
To test a treatment's effect on a sick population,
subjects are recruited and admitted into the trial based on 
eligibility criteria designed to identify sick subjects.
The trial subjects are then randomly split into a treated group that receives the treatment and a control group that receives the best alternative treatment (or a placebo). 
%randomly split the eligible subjects, giving half of 
%them the treatment and the other
%half the previous best treatment (the placebo).
``Targeted" means only sick individuals are admitted
into the trial via the eligibility criteria, with the implicit assumption that only a single treatment-effect is to be estimated.
This ignores the possibility of  
subgroups among the treated population with heterogeneous treatment
effects.  
Further, one often does not have the luxury of a targeted
RCT. For example, eligibility criteria for admittance to the trial may not unambiguously identify 
sick subjects, or one may not be able to control who gets into the trial. 
When the treatment is not exclusively applied on sick subjects, we say the trial is non-targeted and 
new methods are needed to extract the treatment effect on the sick,
see \cite{zhang2021subgroup}. 
Non-targeted trials are
the norm whenever subjects self-select into an intervention,
which is often the case
across domains stretching from healthcare to advertising
(see for example \cite{malik235}).
We propose a nonparametric approach to causal inference in non-targeted trials,
based on a pre-cluster and merge strategy.

Assume a population is broken into \math{\ell} groups with different expected treatment effects in each group. Identify each group with the level of its treatment effect, so there are effect levels \math{c=0,1,\ldots,\ell-1}. For example, a population's subjects
can be healthy, \math{c=0}, or sick, \math{c=1}.
We use the Rubin-Neyman potential outcome framework \citep{rubin2005causal}. 
A subject is a tuple \math{s=(x,c,t,y)}  sampled from a distribution \math{D}, where \math{x\in[0,1]^d}
is a feature-vector such as [age, weight], \math{c} indicates the subject's level, \math{t} indicates the subjects treatment cohort, 
and \math{y} is the observed outcome.
The observed outcome is one of two potential outcomes,
\math{v} if treated or \math{\vbar} if not treated. 
We 
consider strongly ignorable trials: 
given \math{x}, the propensity to treat is
strictly between 0 and 1 and
the potential outcomes \math{\{v,\vbar\}} depend only on \math{x}, independent of \math{t}.
In a strongly ignorable trial, one can use the features to 
identify counterfactual controls for estimating effect.
The effect level~\math{c} is central to the scope
of our work. Mathematically, \math{c} is a hidden effect
modifier which determines the distribution of the potential outcomes
(\math{c} is an unknown and possibly complex function of \math{x}). 
The effect level~\math{c} partitions the
feature space into subpopulations with
different effects.
One tries to design the eligibility criteria for the trial to
ensure that the propensity to treat is non-zero only for subjects in one level. What to do when the eligibility criteria allow more than one level into the trial is exactly the
problem we address. Though our work applies to a general number of levels, all the main ideas can be illustrated with just two levels,
\math{c\in\{0,1\}}. For the sake of concreteness,
we denote these two levels healthy and sick.

A trial samples \math{n} subjects,
\math{s_1,\ldots,s_n}. If 
subject \math{i} is treated, \math{t_i=1} and the observed outcome 
\math{y_i=v_i}, otherwise \math{t_i=0}, and the observed
outcome is \math{\vbar_{i}} (consistency).
The
treated group is \math{\cl T=\{i\given t_i=1\}}, the control
group is \math{\cl C=\{i\given t_i=0\}}, and the sick
group is \math{\cl S=\{i\given c_i=1\}}.
The task is to estimate
if the treatment works on the sick, and if there is any detriment to
the healthy,
\eqar{
\effect&=&\Exp_D[v-\vbar\given c=1]\\
\sideeffect&=&\Exp_D[v-\vbar\given c=0].\nonumber
}
Most prior work estimates \effect{} using the average treatment effect for the treated, the \att~\citep{abdia2017propensity},
\mld{
\att=
\text{average}_{i\in\cl T}(v_i)
-
\text{average}_{i\in\cl T}(\vbar_i),
\label{eq:att}
}
which assumes all treated subjects are sick.
There are several complications with this approach.
\begin{enumerate}[label={(\roman*)},itemsep=3pt,topsep=5pt,leftmargin=19pt]
\item Suppose a subject is treated with probability
\math{p(x,c)}, the propensity to treat.
For a non-uniform propensity to treat,
the treated group has a selection bias, and
\att{} is a biased estimate of~\effect.
Ways to address this bias 
include inverse propensity weighting \citep{rosenbaum1983central},
matched controls \citep{abdia2017propensity}, and learning the outcome 
function \math{y(x,t)}, see for example \cite{alaa2017deep,athey2019generalized, hill2011bayesian,kunzel2019metalearners, vansteelandt2014regression, wager2018estimation}.
Alternatively, one can simply ignore this bias and accept that
\att{} is estimating \math{\Exp[v-\vbar\given t=1]}.

\item The second term on the RHS in \r{eq:att} can't be computed
because
we don't know the counterfactual~\math{\vbar} for 
treated subjects.
 Much
of causal inference deals with accurate unbiased estimation of 
\math{\text{average}_{i\in\cl T}(\vbar_i)}
\citep{bang2005doubly, hainmueller2012entropy}. Our goal is not
counterfactual estimation. Hence, in our experiments, we use off-the-shelf counterfactual estimators.

\item (\emph{Focus of our work}) The trial is non-targeted and some (often most) treated subjects are healthy.

\end{enumerate}
To highlight the challenge in (\rn{3}) above, 
consider a simple case with uniform propensity to treat,
\math{p(x,c)=p}. Conditioning on at least one treated subject,
\mand{\Exp[\att]=\Prob[\text{sick}]\times\effect+
\Prob[\text{healthy}]\times\sideeffect.}
The \att{} is a mix of effects and is therefore biased 
when the treatment effect is heterogeneous across subpopulations. 
In many settings, for example healthcare, \math{\Prob[\text{sick}]\ll\Prob[\text{healthy}]} and the bias is extreme, rendering \att{} useless.
Increasing the number
of subjects won't resolve this bias. State-of-the-art causal inference packages
provide methods to compute \att{}, specifically aimed at accurate
estimates of the counterfactual \math{\text{average}_{i\in\cl T}(\vbar_i)}
\citep{econml, sharma2020dowhy}. These packages
suffer from the mixing bias above. We propose a fix which
can be used as an add-on to these packages. That is, we do not provide
better methods for counterfactual estimation. We provide a new tool that
uses methods from counterfactual estimation
to provably extract the correct group/sub-population effects.

\paragraph{Our Contribution.}
Let us first emphasize that we do not contribute to counterfactual
estimation, one of the grand challenges in causal inference. That is,
we assume for each treated individual \math{i}, an unbiased estimate
of the counterfactual \math{\vbar_i} is available. Our contribution is
to correctly identify effects when multiple effect-levels are present
in the treated population.
Our main result is an asymptotically consistent distribution independent
 algorithm 
 to extract the correct effect levels and associated subpopulations in non-targeted trials, when the number of effect-levels is \emph{finite} but
 \emph{unknown}, a setup which occurs frequently in practice.
 Our main result is Theorem~\ref{theorem:main} which makes
 a claim about our PCM algorithm in Section~\ref{sec:theory-algorithm}.
 Assume a non-targeted trial with \math{\ell} effect levels and
 a treated group with \math{n} subjects sampled i.i.d.
 from an \emph{unknown} distribution \math{D}.
Our algorithm identifies \math{\hat\ell} effect-levels with estimated
expected effect \math{\hat\mu_c} in level \math{c}, and assigns each subject \math{s_i} to a level \math{\hat c_i} which,
under mild technical conditions, satisfies:
\begin{theorem}\label{theorem:main}
All of the following hold with probability \math{1-o(1)}:
\begin{enumerate}[label={(\arabic*)},itemsep=3pt,topsep=5pt]
\item \math{\hat\ell=\ell}, i.e., the correct number of effect levels \math{\ell} is identified.
\item \math{\hat\mu_c=\Exp[v-\bar v\given c]+o(1)}, i.e., the effect at each level is estimated accurately.
\item The fraction of subjects assigned the correct effect level is \math{1-o(1)}. The effect level 
\math{\hat c_i} is correct
if \math{\mu_{\hat c_i}} matches,
to within \math{o(1)}, the expected treatment effect for the subject.
\end{enumerate}
\end{theorem}
For the formal assumptions needed to prove Theorem~\ref{theorem:main},
see Section~\ref{sec:theory-consistency}.
Parts (1) and (2) say the algorithm extracts the correct number of levels and their expected effects.
Part (3) says the correct subpopulations for each level are extracted. Knowing the correct subpopulations is
useful for post processing, for example to 
understand the effects in terms of the features.
The main challenge in achieving the results in
Theorem~\ref{theorem:main} is that the feature-values of the
subpopulation corresponding to a given expected effect (an effect level)
may be spread over
the feature space in an unknown way. That is, there are multiple clusters
in the feature space corresponding to a given expected effect. Hence it is
necessary to first find these clusters (the pre-clustering step in our
algorithm) and then merge all the clusters of a single effect-level.
Since the observed effects conditioned on an
effect-level come from a continuous distribution it now becomes challenging
to identify which clusters belong to which effect level, or even
just how many effect levels there are and the corresponding expected
effect in each level.

The above challenges are all addressed in 
our algorithm satisfying Theorem~\ref{theorem:main} given in Section~\ref{sec:theory-algorithm}.
We use
an unsupervised pre-cluster and merge strategy
which reduces the task of estimating the effect-levels to a 1-dimensional
optimal clustering problem that provably extracts the correct
levels asymptotically as 
\math{n\rightarrow\infty}. 
Our algorithm assumes an unbiased estimator of counterfactuals,
for example some established method~\citep{econml, sharma2020dowhy}. 
In practice, one estimates conterfactuals by identifying
control subjects after controlling
for confounders.
If unbiased counterfactual estimation is not possible, then
causal effect analysis becomes a challenge due to this counterfactual
bias,
irrespective of whether multiple effect-levels are in the treated group. 
Our primary goal is untangling the heterogeneous effect levels.
We use 
an off-the-shelf gradient boosting algorithm to get
counterfactuals in our experiments~\citep{econml}. 
We demonstrate our algorithm's performance
on synthetic data, in particular to show that practice
matches the theory.  The goal of our
experiments is to showcase our algorithm's capabilities in
comparison to alternatives.
On a more practical side, our algorithm has been used to extract
group effects in real medical data~\citep{malik235}, which is not the focus
here. We do mention that one can relax the requirement of unbiased 
counterfactuals if a large number of control subjects with measured
outcomes is available. One can use our algorithm as is, by setting
the effect in a cluster as the difference in average outcome
between that cluster's treated and control subjects.

Subpopulation effect-analysis is a special case of
heterogeneous treatment effects/conditional average
treatment effects (HTE/CATE), see
\citep{kunzel2019metalearners,shalit2017estimating,wager2018estimation,chernozhukov2022generic}.
Hence, in our experiments,  we compare with X-Learner, a
state-of-the art algorithm for HTE~\citep{kunzel2019metalearners}.
We also compare with the Bayes optimal prediction of effect-level which uses
knowledge of the effect distributions (unrealizable in practice).  
In comparison to X-Learner, our algorithm extracts visually better subpopulations, and has an accuracy that is
more than \math{10\times} better for estimating per-subject expected effects.
Note, HTE algorithms do not extract subpopulations with effect-levels. They 
predict effect given the features \math{x}. One can, however,
try to infer subpopulations from predicted effects as in the GATES method in
\cite{chernozhukov2022generic}. However, as mentioned in
\cite{chernozhukov2022generic}, proving consistency
is difficult and would require strong assumptions because consistency of
the ML method for the much harder CATE estimation is needed.
We go directly after the subpopulation effect-levels.
Our algorithm also significantly outperforms Bayes optimal
based on individual effects, which suggests
that some form of pre-cluster and merge strategy is necessary.
This need for some form of clustering has been independently
observed in \cite[chapter 4]{kim2020causal}
who studies a variety of clustering approaches in a
non-distribution independent setting with a known number of levels.

\section{Algorithm: Pre-Cluster and Merge For Subpopulation Effects (PCM)}
\label{sec:theory-algorithm}

Our algorithm uses a nonparametric 
pre-cluster and merge strategy that achieves asymptotic consistency without any user-specified hyperparameters. The inputs are the
\math{n} subjects
\math{s_1,\ldots,s_n}, where
\mandc{\{s_i\}_{i=1}^n=\{(x_i,t_i,y_i,\bar y_i)\}_{i=1}^n.}
Note, both the factual \math{y_i} 
and counterfactual \math{\bar y_i} are inputs to the algorithm,
where
\math{y_i} and
\math{\bar y_i} are random draws from the distributions of
\math{v_i} and \math{\vbar_i}.
In practice, one measures \math{y_i} and one
gets an unbiased estimate of
\math{\bar y_i} from a counterfactual estimation
algorithm that uses the observed features \math{x_i} to control for
possible confounders.
For our
demonstrations we use an out-of-the-box gradient boosting regression algorithm  
to estimate counterfactuals
\citep{friedman2002stochastic, pedregosa2011scikit}.
Inaccuracy in counterfactual estimation will be accommodated in our analysis.
The need to estimate counterfactuals does
impact the algorithm in practice,
due to an asymmetry in 
most trials: the treated population is much smaller
than the untreated controls. Hence, one might be able to 
estimate counterfactuals for the treated population but not 
for the controls due to lack of coverage by the (small) treated population. In this case, our algorithm is only run on the treated population.
It is convenient to define individual treatment effects 
\math{\ite_i=(y_i-\bar y_i)(2t_i-1)}, where \math{y_i} is the observed factual and
\math{\bar y_i} the counterfactual
(\math{2t_i-1=\pm 1} ensuring that the effect
computed is for treatment versus no treatment).
There are five main steps.
\begin{center}
\fbox{
\parbox{0.9\linewidth}{
\begin{algorithmic}[1]\itemsep3pt
\STATE {\sc[pre-cluster]} Cluster the \math{x_i} into \math{K\in O(\sqrt{n})} clusters
\math{Z_1,\ldots,Z_{K}}.
\STATE Compute \att{} for each cluster \math{Z_j},
\math{\att_j=\text{average}_{x_i\in Z_j}\ite_i.\remove{\vspace*{-10pt}}}
\STATE {\sc[merge]} Group the \math{\{\att_j\}_{j=1}^K} into \math{\hat\ell} effect-levels, merging the clusters at each level to get subpopulations
\math{X_0, X_1,\ldots, X_{\hat\ell-1}}. (\math{X_c} is the union of all clusters at level \math{c}.)
\STATE Compute subpopulation effects
\math{\hat\mu_c=\text{average}_{x_i\in X_c} \ite_i}, for \math{c=0,\ldots,\hat\ell-1}.
%and return \math{(C_0,\effect_0), (C_1,\effect_1), \ldots}.
\STATE Assign subjects to effect levels,
update the populations \math{X_{c}} and expected effects \math{\hat\mu_{c}}.
\end{algorithmic}
}}
\end{center}
We now elaborate on the intuition and details for each step in the
algorithm.
\begin{enumerate}[nolistsep,topsep=5pt,itemsep=3pt,label={{\bf Step \arabic*.}}, itemindent=28pt, leftmargin=12pt]
\item 
The clusters in the pre-clustering step play two roles. The first is to denoise individual 
effects using in-cluster averaging. The second is to group like with
like, that is clusters should be homogeneous, containing only
subjects from one effect-level. This means each 
cluster-\att{} will accurately estimate a single level's effect (we do not know
the level). We allow for any clustering algorithm. However, our theoretical
analysis (for simplicity) uses a specific algorithm,
box-clustering, based on an \math{\varepsilon}-net of the feature
space. One could also use a standard clustering algorithm such as \math{K}-means.
We compare box-clustering with \math{K}-means in the appendix.

\item Denoise individual effects using in-cluster averages.
Assuming clusters are homogeneous, a cluster's \att{} approximates
some level's effect. In essence, each cluster is a targeted trial.
We don't know which effect level the cluster is targeting, though.

\item 
  Assuming the effects in different levels are well separated, this
  separation is present in the cluster-\att{}s,
  provided clusters are homogeneous.
  Hence, we identify and merge clusters with similar effects
  into subpopulations.
  Two tasks must be solved. Finding the number of subpopulations
  \math{\hat\ell} and then
  optimally grouping the clusters into
  \math{\hat\ell} subpopulations.
  To find the subpopulations, 
  we use \math{\hat\ell}-means with squared 1-dim clustering error.
  Our algorithm sets \math{\hat\ell}
  to achieve an \math{\hat\ell}-means error at most
  \math{\log n/n^{1/2d}}. So,
  \eqan{
    \text{optimal 1-dim clustering error}(\hat\ell-1)&>&\log n/n^{1/2d}
    \\
    \text{optimal 1-dim clustering error}(\hat\ell)&\le&\log n/n^{1/2d}
  }
  Simultaneously finding \math{\hat\ell} and optimally partitioning the
  clusters into \math{\hat\ell} groups
  can be solved using a standard dynamic programming algorithm in \math{O(K^2\hat\ell)} time using \math{O(K)} space~\citep{wang2011ckmeans}.
  Note, while our algorithm  identifies the number of effect 
  levels, if it is known that only two distinct
  subpopulations exist, sick and healthy,
  then \math{\hat\ell} can be hard-coded to 2.

\item Assuming each cluster is homogeneous, merging the
  clusters with similar effects found in step 3 will form subpopulations
  that are near-homogeneous, containing subjects from
  just one effect-level. Hence, the subpopulation-\att{}s will be accurate
estimates of the effects at each level.

\item Each subject \math{x_i} is implicitly 
assigned a level \math{\hat c_i} based on the
subpopulation \math{X_{c}} to which it belongs. However, we can do
better. By considering the \math{\sqrt{n}} nearest neighbors to 
\math{x_i}, we can obtain a smoothed effect for \math{x_i}. We use this smoothed effect to place \math{x_i} into the subpopulation whose effect
matches best, hence placing \math{x_i} into a level.
Running this update for all \math{n} subjects is costly,
needing sophisticated data structures to reduce the expected run time below
\math{O(n^2)}. As an alternative, we center an \math{(1/n^{1/2d})}-hypercube
on \math{x_i} and smooth \math{x_i}'s effect using the average effect over points in this hypercube. This approach requires 
\math{O(n\sqrt{n})} run time to obtain the effect-level for all 
subjects, significantly better than \math{O(n^2)} when \math{n} is large.
Once the effect-levels for all subjects are obtained, one can update the 
subpopulations 
\math{X_c} and the corresponding effect-estimates \math{\hat\mu_c}.
We used this last approach in our experiments.
\end{enumerate}
The run time of the algorithm is \math{O(n\ell+n\sqrt{n})} (expected and with high probability) and the
%KPB dividing into two paragrphas one on steps, one on post processing
output is nearly homogeneous subpopulations which can now be post-processed.
An example of useful post-processing is a feature-based explanation of the
subpopulation-memberships as in ~\citep{malik235}.
Note that we still do not know which subpopulation(s) are the sick ones,
hence we cannot say which is the effect of the treatment on the ``sick''. A post-processing oracle would make this determination. For example, a doctor in a medical trial could identify the sick groups from subpopulation-demographics.

{\bf Note.} The optimal 1-d clustering can be done directly on the
smoothed \ite{}s from the \math{(1/n^{1/2d})}-hypercubes
centered on each \math{x_i}, using the same thresholds in step 3. One still gets asymptotic consistency, however the price is an increased run time to \math{O(n^2\ell)}. This 
is prohibitive for large \math{n}.

\paragraph{Extension to Accomodate No Counterfactual Estimator.}
If an unbiased estimator of counterfactuals is not available but untreated
controls with measured outcomes
are available, our pre-cluster and merge methodology can still be used.
Now, within each cluster, one estimates the average treatment outcome and the
average control outcome. Define the \math{\att_j} for
cluster \math{j} as the difference between these two average outcomes. We would
need additional technical assumptions to ensure that this difference
of average outcomes converges to the cluster \math{\att_j}, but thereafter
all the same results hold.

\section{Asymptotic Consistency: Proof of Theorem~\ref{theorem:main}}
\label{sec:theory-consistency}

To prove consistency, we must make our assumptions precise.
In some cases the assumptions are stronger than needed, for simplicity of exposition.

\begin{enumerate}[nolistsep,topsep=5pt,itemsep=3pt, label={{\bf A\arabic*.}}, itemindent=14pt, leftmargin=16pt]
\item 
The feature space \math{X} is 
\math{[0,1]^d} and the marginal feature-distribution is 
uniform, \math{D(x)=1}. More generally, \math{X} is compact and \math{D(x)} is bounded, \math{0<\delta\le D(x)\le\Delta} (can be relaxed). We assume the
features of each subject are sampled i.i.d. from \math{D}.

\item The level \math{c} is an unknown
function of the feature
\math{x}, \math{c=h(x)}. The potential effects \math{(v,\vbar)}
depend only on \math{c}. Conditioning on \math{c},
effects are well separated.
Let 
\math{\mu_c=\Exp_D[v-\vbar|c]}. Then,
\mandc{
|\mu_c-\mu_{c'}|\geq\kappa
\qquad\text{for \math{c\not=c'}}
}
\item
Define the subpopulation for level
\math{c} as \math{X_c=h^{-1}(c)}. Each subpopulation
has positive measure, \math{\Prob[x\in X_c]=\beta_c\ge\beta>0}.

\item For a treated subject \math{x_i} with outcome \math{y_i},
  one can get
  an unbiased estimate of the counterfactual \math{\bar y_i}.
  Effectively, an unbiased
estimate of the individual treatment effect
\math{\ite{}_i=y_i-\bar y_i} is available. 
%Note that by making stronger assumptions
%on the outcomes \math{y_i,\bar y_i} one can relax the need for counterfactual estimates. 
Any causality analysis 
requires estimates of counterfactuals.
In practice, one typically gets counterfactuals from
untreated subjects
after controlling for confounders~\citep{econml, sharma2020dowhy}.

\item Sample 
averages concentrate. Essentially, the estimated
\ite{}s are independent. This is true in practice because the subjects are independent and the counterfactual estimates use a predictor learned from the independent control population.  
For \math{m} i.i.d. subjects, let the average of the estimated \ite{}s be \math{\hat\nu}
and the expectation of this average be \math{\nu}.
Then, 
\mandc{
\Prob[|\hat\nu-\nu|>\epsilon]\leq e^{-\gamma m\epsilon^2}.
}
The parameter \math{\gamma>0} is related to distributional properties of the estimated \ite{}s. Higher variance \ite{} estimates result
in \math{\gamma} being smaller. Concentration is a mild technical assumption
requiring the estimated effects to be unbiased well behaved random variables, to which a central limit theorem applies.
Bounded effects or normally distributed effects suffice for concentration.
Note also that \math{\gamma} will be impacted by the
variance in the counterfactual estimation.

\item The boundary between the subpopulations 
has small measure. Essentially we require that
two subjects that have very similar features will belong to the same level with high probability (the function \math{c=h(x)} is
not a ``random'' or fractal function). 
Again, this is a mild technical assumption which is taken for granted in practice. Let us make the assumption more precise. 
%A point \math{x} is a boundary point if every ball of positive radius centered on 
%\math{x} contains points from more than level.
Define an \math{\varepsilon}-net to be a subdivision of \math{X} into \math{(1/\varepsilon)^d} disjoint
hypercubes of side \math{\varepsilon}. A hypercube of an \math{\varepsilon}-net is impure if it contains points from multiple subpopulations. Let \math{N_{\text{impure}}} be the number of impure hypercubes in an \math{\varepsilon}-net. Then 
\math{\varepsilon^d N_{\text{impure}}\le \alpha \varepsilon^\rho},
where \math{\rho>0} and \math{\alpha} is a constant.
Note, this assumption is equivalent to assuming that 
\math{d-\rho} is the boxing-dimension
of the boundary. The norm in practice is  \math{\rho=1}. 

\item We use box-clustering for the first step in the algorithm. Given \math{n}, define \math{\varepsilon(n)=1/\floor{n^{1/2d}}}. All points in a
hypercube of an \math{\varepsilon(n)}-net form a cluster. Note that
the number of clusters is approximately \math{\sqrt{n}}. The 
expected number of points in a cluster is 
\math{n\varepsilon(n)^d\approx \sqrt{n}}.
\end{enumerate}

\paragraph{Discussion of Assumptions.} Assumptions
{\bf A1-A3 and A6} are benign technical assumptions
that are generally true in
practice. The main practical assumption is
{\bf A4} that counterfactual estimation is possible. The accuracy of any
causal analysis relies on the accuracy of counterfactual estimation,
which is not our focus.
{\bf A5} is a non trivial technical assumption which is usually true in
practice. {\bf A7} is just the simplest clustering algorithm for analysis.
Our results can hold for any other consistent clustering approach, and in
our experiments we find that \math{\sqrt{n}}-means clustering is a
simple algorithm with good performance.

\subsection{Proof of Theorem~\ref{theorem:main}}
\label{sec:theory-consistency-proof}

We prove Theorem~\ref{theorem:main} via a sequence of lemmas.
The feature space $X=[0,1]^d$ is partitioned into levels
\math{X_0,\ldots,X_{\ell-1}}, where \math{X_c=h^{-1}(c)} is the set of points whose level is \math{c}. Define an
\math{\varepsilon}-net that partitions $X$ into 
$N_\varepsilon=\varepsilon^{-d}$ hypercubes of equal volume  $\varepsilon^d$, where $\varepsilon$ is the side-length of the hypercube.
Set \math{\varepsilon=1/\floor{n^{1/2d}}}. Then,  
\math{N_\varepsilon=\sqrt{n}(1-O(d/n^{1/2d}))\sim\sqrt{n}}. Each hypercube in the \math{\varepsilon}-net defines a cluster for the pre-clustering
stage. There are about \math{\sqrt{n}} clusters and, since \math{D(x)} is uniform, there are about \math{\sqrt{n}} points in each cluster.
Index the clusters in the \math{\varepsilon}-net by \math{j\in\{1,\ldots,N_\varepsilon\}} and
define \math{n_j} as the number of points in cluster \math{j}.
Formally, we have,
\begin{lemma}\label{lemma:points}
Suppose $D(x)\ge\delta>0$. Then, 
\math{\Prob[\min_j n_j\ge \frac12\delta\sqrt{n}]>1- \sqrt{n}\exp(-\delta\sqrt{n}/8)}.
\end{lemma}
\begin{proof}
Fix a hypercube in the \math{\varepsilon}-net. Its volume is 
\math{\varepsilon^d\ge(1/n^{1/2d})^d=1/\sqrt{n}}.  A point lands in this 
hypercube  with probability at least \math{\delta/\sqrt{n}}. Let \math{Y} be the number of points 
in the hypercube. Then, \math{Y} is a sum of \math{n} independent Bernoullis and 
\math{\Exp[Y]\ge\delta\sqrt{n}}. By a Chernoff bound~\cite[page 70]{motwani1996randomized},
\mand{
\Prob[Y<\delta\sqrt{n}/2]\le\Prob[Y<\Exp[Y]/2]<\exp(-\Exp[Y]/8)\le\exp(-\delta\sqrt{n}/8).
}
By a union bound over the \math{N_\varepsilon} clusters,
\mandc{
\Prob[\text{some cluster has fewer than \math{\delta\sqrt{n}/2} points}]
<
N_\varepsilon\exp(-\delta\sqrt{n}/8)
\le
\sqrt{n}\exp(-\delta\sqrt{n}/8).
}
The lemma follows by taking the complement event.
\end{proof}
For uniform \math{D(x)}, \math{\delta=1} and every cluster has at least 
\math{\frac12\sqrt{n}} points with high probability. We can now condition on this 
high probability event that every cluster is large. This means that a cluster's 
\att{} is an average of many \ite{}s, which by {\bf A5} concentrates at the 
expected effect for the hypercube.
Recall that the expected effect in level \math{c} is defined as 
$\mu_c = \Exp_D[v-\vbar|c]$. We can assume, w.l.o.g., that \math{\mu_0<\mu_1\cdots<\mu_{\ell-1}}.  
Define \math{\nu_j} as the expected average effect for points in the hypercube \math{j}
and \math{\att_j} as the average \ite{} for points in cluster \math{j}.
Since every cluster is large, every cluster's \math{\att_j} will be close to its
expected average effect \math{\nu_j}. More formally,
\begin{lemma}\label{lemma:concentrate}
\math{\Prob[\max_j|\att_j-\nu_j|\le 2\sqrt{\log n/\gamma\delta\sqrt{n}}]\ge 1-n^{-3/2}-\sqrt{n}\exp(-\delta\sqrt{n}/8)}.
\end{lemma}
\begin{proof}
Conditioning on \math{\min_jn_j\ge\frac12\delta\sqrt{n}} and using {\bf A5}, we have
\mand{\Prob\Big[|\att_j-\nu_j|>2\sqrt{\log n/\gamma\delta\sqrt{n}}\Big|  \min_jn_j\ge{\textstyle\frac12}\delta\sqrt{n}\Big]\le 
\exp(-2\log n)=1/n^2.}
By a union bound,
\math{\Prob[\max_j|\att_j-\nu_j|>2\sqrt{\log n/\gamma\delta\sqrt{n}}\given \min_jn_j\ge{\textstyle\frac12}\delta\sqrt{n}]\le 
N_\varepsilon/n^2.} 
For any events \math{\cl{A},\cl{B}}, by total probability,
\math{
\Prob[\cl{A}]\le\Prob[\cl{A}\given \cl{B}]+\Prob[\overline{\cl{B}}].
}
Therefore,
\mand{
\Prob[\max_j|\att_j-\nu_j|>2\sqrt{\log n/\gamma\delta\sqrt{n}}]
\le
N_\varepsilon/n^2
+
\Prob[\min_jn_j<{\textstyle\frac12}\delta\sqrt{n}]
}
To conclude the proof,
use \math{N_\varepsilon\le\sqrt{n}} and Lemma~\ref{lemma:points}. 
\end{proof}
Lemmas~\ref{lemma:points} and~\ref{lemma:concentrate}
are standard concentration results. The remaining
lemmas to identify the levels and subpopulations at each level is done by
reducing the problem to one of optimal 1-dimensional clustering which
can be solved efficiently.
A hypercube in the \math{\varepsilon}-net is homogeneous if it only contains points
of one level (the hypercube does not intersect the boundary between levels).
Let \math{N_c} be the number of homogeneous hypercubes for level \math{c}
and
\math{N_{\text{impure}}} be the number of hypercubes that are not homogeneous,
i.e., impure.
\begin{lemma}\label{lemma:pure}
\math{N_{\text{impure}}\le \alpha\varepsilon^\rho N_{\varepsilon}}
and
\math{N_c\ge N_{\varepsilon}(\beta/\Delta-\alpha\varepsilon^\rho)}.
\end{lemma}
\begin{proof}
{\bf A6} directly implies 
\math{N_{\text{impure}}\le \alpha\varepsilon^\rho N_{\varepsilon}}.
Only the pure level \math{c} or impure hypercubes can contain points in level \math{c}. Using {\bf A3} and 
\math{\varepsilon^d=1/N_\varepsilon}, we have
\mand{\beta\le \Prob[x\in X_c]\le (N_c+N_{\text{impure}})\Delta\varepsilon^d
\le (N_c+\alpha\varepsilon^\rho N_\varepsilon)\Delta/N_\varepsilon.
}
The result follows after rearranging the above inequality.
\end{proof}

The main tools we need are Lemmas~\ref{lemma:concentrate} and \ref{lemma:pure}. 
Let us recap what we have. The cluster \att{}s are close to the 
expected average effect in every hypercube. The number of 
impure hypercubes is an asymptotically negligible fraction of the
hypercubes since \math{\varepsilon\in O(1/n^{1/2d})}. Each level has
an asymptotically constant fraction of homogeneous hypercubes.
This means that almost all cluster \att{}s will be close to a 
level's expected effect, and every level will be well represented.
Hence, if we optimally cluster the \att{}s, with fewer than
\math{\ell} clusters, we won't be able to get clustering error 
close to zero. With at least \math{\ell} clusters, we will be able to get clustering error approaching zero. This is the content of the next lemma, which justifies step 3 in the algorithm. An optimal
\math{k}-clustering of the cluster \att{}s produces 
\math{k} centers \math{\theta_1,\ldots,\theta_k}
and assigns each cluster \math{\att_j}
to a center \math{\theta(\att_j)} so that the average
clustering error 
\math{\text{err}(k)=\sum_{j}(\att_j-\theta(\att_j))^2/N_\varepsilon} is minimized.
Given \math{k}, one can find an optimal \math{k}-clustering in 
\math{O(N_\varepsilon^2k)} time using \math{O(N_\varepsilon)} space.
\begin{lemma}\label{lemma:clustering} 
With probability at least 
\math{1-n^{-3/2}-\sqrt{n}\exp(-\delta\sqrt{n}/8)},
optimal clustering of 
the \att{}s with \math{\ell-1} and \math{\ell} clusters produces clustering errors which satisfy
\eqan{
\text{err}(\ell-1)&\ge&(\beta/\Delta-\alpha\epsilon^\rho)\Big(\kappa/2-2\sqrt{\log n/\gamma\delta\sqrt{n}}\Big)^2\qquad\qquad\text{for \math{\textstyle\frac{\log n}{\sqrt{n}}<\frac{\kappa^2\gamma\delta}{16}}} 
\\
\text{err}(\ell)&\le&
{\textstyle\frac14}\alpha\varepsilon^\rho(\mu_{\ell-1}-\mu_0)^2
+
4\log n(1+\alpha\varepsilon^\rho)/\gamma\delta\sqrt{n}
}
\end{lemma}

\begin{proof}
With the stated probability, by Lemma~\ref{lemma:concentrate},
all \att{}s are within 
\math{2\sqrt{\log n/\gamma\delta\sqrt{n}}} of the expected effect for their respective hypercube. This, together with 
Lemma~\ref{lemma:pure} is enough to prove the bounds.

First, the upper bound on \math{\text{err}(\ell)}. Choose
cluster centers \math{\mu_0,\ldots,\mu_{\ell-1}}, the 
expected effect for each level. This may not be optimal, so it gives
an upper bound on the cluster error. Each homogeneous hypercube 
has a expected effect which is one of these levels, and its 
\att{} is within \math{2\sqrt{\log n/\gamma\delta\sqrt{n}}} of the
corresponding \math{\mu}. Assign each \att{} for a homogeneous 
hypercube to
its corresponding \math{\mu}. The homogeneous hypercubes have 
total clustering error at most 
\math{4\log n(N_\varepsilon-N_{\text{impure}})/\gamma\delta\sqrt{n}}.
For an impure hypercube, the expected average effect is a convex combination of \math{\mu_0,\ldots,\mu_{\ell-1}}. Assign these
\att{}s to either \math{\mu_0} or \math{\mu_{\ell-1}}, with an error at most
\math{(2\sqrt{\log n/\gamma\delta\sqrt{n}}+\frac12(\mu_{\ell-1}-\mu_0))^2}. Thus,
\eqan{
N_\varepsilon\text{err}(\ell)
&\le&
\frac{4\log n(N_\varepsilon-N_{\text{impure}})}{\gamma\delta\sqrt{n}}
+
N_{\text{impure}}(2\sqrt{\log n/\gamma\delta\sqrt{n}}+{\textstyle\frac12}(\mu_{\ell-1}-\mu_0))^2
\\
&\le&
\frac{4\log n(N_\varepsilon+N_{\text{impure}})}{\gamma\delta\sqrt{n}}
+
\frac{N_{\text{impure}}(\mu_{\ell-1}-\mu_0)^2}{2}
}
The upper bound follows after dividing by \math{N_\varepsilon} and using 
\math{N_{\text{impure}}\le\alpha\varepsilon^\rho N_\varepsilon}.

Now, the lower bound on \math{\text{err}(\ell-1)}. Consider any \math{\ell-1}
clustering of the \att{}s with centers \math{\theta_0,\ldots,\theta_{\ell-2}}. At least \math{N_c\ge N_\varepsilon(\beta/\Delta-\alpha\epsilon^\rho)} of the \att{}s are within 
\math{2\sqrt{\log n/\gamma\delta\sqrt{n}}} of \math{\mu_c}. We also know that
\math{\mu_{c+1}-\mu_c\ge\kappa}. Consider the \math{\ell} disjoint intervals 
\math{[\mu_c-\kappa/2,\mu_c+\kappa/2]}.
By the pigeonhole principle, at least one of these intervals \math{[\mu_{c*}-\kappa/2,\mu_{c*}+\kappa/2]}
does not contain
a center. Therefore all the \att{}s associated to \math{\mu_{c*}} will incur an
error at least \math{\kappa/2-2\sqrt{\log n/\gamma\delta\sqrt{n}}} when
\math{\kappa/2>2\sqrt{\log n/\gamma\delta\sqrt{n}}}.
The total 
error is
\mand{
N_\varepsilon\text{err}(\ell-1)
\ge
N_{c*} \Big(\kappa/2-2\sqrt{\log n/\gamma\delta\sqrt{n}}\Big)^2.
}
Using \math{N_{c*}\ge N_\varepsilon(\beta/\Delta-\alpha\epsilon^\rho)}
and dividing by \math{N_\varepsilon} concludes the proof.
\end{proof}
Lemma~\ref{lemma:clustering} is crucial to estimating the number of
levels. The error is at least
\math{\beta\kappa^2/4\Delta(1+o(1))} for fewer than 
\math{\ell} clusters and at most 
\math{\frac14\alpha\varepsilon^\rho(\mu_{\ell-1}-\mu_0)^2(1+o(1))} for 
\math{\ell} or more clusters. Any function \math{\tau(n)} that asymptotically separates 
these two errors can serve as an error threshold. 
The function should be agnostic to the parameters \math{\alpha,\beta,\kappa,\Delta,\rho,\ldots}. In practice, \math{\rho=1} and since
\math{\varepsilon\sim 1/n^{1/2d}}, we have chosen 
\math{\tau(n)=\log n/n^{\rho/2d}}. Since \math{\text{err}(\ell-1)} is asymptotically
constant, \math{\ell-1} clusters can't achieve error 
\math{\tau(n)} (asymptotically). Since \math{\text{err}(\ell)\in O(\varepsilon^\rho)},
\math{\ell} clusters can achieve error \math{\tau(n)} (asymptotically).
Hence, choosing \math{\hat\ell} as the minimum number of clusters that 
achieves error \math{\tau(n)} will asymptotically output 
the correct number of clusters \math{\ell}, with high probability,
proving part~(1) of Theorem~\ref{theorem:main}.

We now prove parts (2) and (3) of  Theorem~\ref{theorem:main}, which follow from the 
accuracy of steps 4 and 5 in the algorithm. We know the algorithm asymptotically selects the correct number of levels with high probability.
We show that each level is populated by mostly the homogeneous clusters of that level.
\begin{lemma} With probability at least 
\math{1-n^{-3/2}-\sqrt{n}\exp(-\delta\sqrt{n}/8)}, asymptotically in \math{n},
all the \math{N_c} \att{}s from the homogeneous hypercubes of level \math{c} are assigned to the same cluster in the optimal clustering, and no \att{}s from a different level's
homogeneous hypercubes is assigned to this cluster.
\end{lemma}
\begin{proof}
Similar to the proof of Lemma~\ref{lemma:clustering}, consider the 
\math{\ell} disjoint intervals 
\math{[\mu_c-\kappa/4,\mu_c+\kappa/4]}. One center 
\math{\theta_c} must be placed in this interval otherwise the clustering
error is asymptotically constant, which is not optimal. All the \att{}s for level 
\math{c} are (as \math{n} gets large) 
more than \math{\kappa/2} away from any other center, and at most \math{\kappa/2} away from \math{\theta_c}, which means all these \att{}s get assigned to \math{\theta_c}. 
\end{proof}
Similar to Lemma~\ref{lemma:points}, we can get a high-probability upper bound of 
\math{a\sqrt{n}} on the
maximum number of points in a cluster.
Asymptotically, the number of points in the impure clusters
is \math{n_{\text{impure}}\in O(\varepsilon^\rho \sqrt{n}N_\varepsilon)}.
Suppose these impure points have expected average effect \math{\mu} (a convex
combination of the \math{\mu_c}'s).
The number of points in level \math{c} homogeneous clusters is
\math{n_c\in\Omega(\sqrt{n}N_\varepsilon)}. Even if all impure points are added to level 
\math{c}, the expected average effect for the points in level \math{c} is 
\mldc{\Exp[\ite\given \text{assigned to level \math{c}}]=\frac{n_{\text{impure}}\mu+n_c\mu_c}{n_{\text{impure}}+n_c}
=\mu_c+O(\varepsilon^\rho).\label{eq:effect1}}
Part (2) of Theorem~\ref{theorem:main} follows from the next lemma after setting
\math{\varepsilon\sim 1/n^{1/2d}} and \math{\rho=1}.
\begin{lemma}\label{lemma:cluster-effect}
Estimate \math{\hat\mu_c} as the average \ite{} for all points assigned to 
level \math{c} (the \math{c}th order statistic of the optimal centers 
\math{\theta_0,\ldots,\theta_{\hat\ell-1}}). Then 
\math{\hat\mu_c=\mu_c+O(\varepsilon^\rho+\sqrt{\log n/n})} with probability 
\math{1-o(1)}.
\end{lemma}
\begin{proof}
Apply a Chernoff bound. We are taking an average of proportional to 
\math{n} points with expectation in 
\r{eq:effect1}. This average will approximate the expectation to within 
\math{\sqrt{\log n/n}} with probability \math{1-o(1)}. The details are very similar to 
the proof of Lemma~\ref{lemma:concentrate}, so we omit them.
\end{proof}
Part (3) of Theorem~\ref{theorem:main} now follows because 
all but the \math{O(\varepsilon^\rho)} fraction of points in the impure 
clusters are
assigned a correct expected effect. An additional fine-tuning 
leads to as
much as \math{2\times} improvement in experiments. For each point, consider the \math{\varepsilon}-hypercube centered on that point. By a Chernoff bound, each of these \math{n} hypercubes has 
\math{\Theta(\sqrt n)} points, as in Lemma~\ref{lemma:points}.
All but a fraction \math{O(\varepsilon^\rho)} of these are impure. 
Assign each point to the center \math{\theta_c} that best matches its 
hypercube-``smoothed'' \ite, giving new subpopulations 
\math{X_c} and corresponding subpopulation-effects \math{\hat\mu_c}. 
This EM-style update can be iterated. 
Our simulations show the results for one E-M update.

\section{Demonstration on Synthetic Data}
\label{subsec:synthetic}

We use a 2-dimensional synthetic experiment with three levels
to demonstrate our pre-cluster and merge algorithm (PCM). Alternatives to
our PCM approach
include state-of-the-art methods that directly predict the effect such as meta-learners, and the Bayes optimal classifier based on \ite{}s. All methods used a base
gradient boosting forest with 400 trees to estimate counterfactuals.
%Both of the alternatives fail, especially when the \ite{}s are noisy. 
The subpopulations in our experiment are shown in 
Figure~\ref{fig:data}, where black is effect-level 0, gray is level 1 and white is level 2.
We present detailed results with \math{n=200}K. 
Extensive results can be found in the appendix. 
Let us briefly describe the two existing benchmarks we will compare against.

\newcommand{\pathdd}{data}
\newcommand{\pathdU}{dataUnif}
\newcommand{\recwW}{0.47\columnwidth}

\begin{figure}[!t]
{\tabcolsep0pt
\begin{tabular}{m{0.35\textwidth}m{0.65\textwidth}}
        \includegraphics[width=0.3\textwidth]{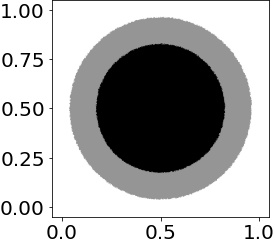}
        &\small
The treatment $t$ is  distributed  randomly between the subjects. The outcome $y$, conditioned on $c$ and $t$, is Gaussian with std. dev. 5:
\mandc{y(t,c)\sim N(\mu_{(t,c)}, 5)}%
The three sub-populations have treatment effects of 0,1,2. 
The expected potential outcome for treatment and level \math{(t,c)} are:
\mand{
\begin{array}{c@{\hspace*{30pt}}c}
\mu_{(0,0)}=0&\mu_{(1,0)}=0,\\ 
\mu_{(0,1)}=0&\mu_{(1,1)}=1,\\
\mu_{(0,2)}=0&\mu_{(1,2)}=2.
\end{array}
}
\end{tabular}}
\caption{Different subpopulations (black, gray, white)
  for synthetic data.}
    \label{fig:data}
\end{figure}

%removed
\remove{
A subject $x$ in our study can belong in either of the three groups $c_2, c_1, c_0 (c(x)=2, c(x)=1, c(x)=0)$. These groups are unknown to us, but play a crucial role in our causal inference analysis, because along with the treatment $t$ assigned to a subject, they determine the subjects outcomes and hence the treatment effect. Imagine an RCT scenario in which the prior probability of a subject to be in $c_i$, $p_{c_i}$, is a lot smaller than $p_{c_j}, \,\, j\neq i$, and at the same time groups $c_j$ have  large positive effects (benefit from the treatment) while $c_i$ group is not susceptible to the treatment (no effect). A standard causal analysis will conclude that the treatment does not work and a simple subgroup analysis will reach  the same conclusion. This mixing bias results in an underestimation of the treatment effect.

This is where our methodology is useful in order to recover the effect groups. In our theory at section ~\ref{sec:theory} we propose a box clustering method to recover homogeneous clusters and we proved that as the number of subjects in our trial increases, we are able to recover higher quality clusters. Of course this type of boxing algorithm could not be used in practise in this form, so we propose that unsupervised clustering algorithms could be used in its place, the precise steps of our methodology are given at algorithm~\ref{alg:algorithm1}
\begin{algorithm}
\DontPrintSemicolon
  \SetAlgoLined
\caption{Effect Untangling Algorithm}\label{alg:group-discovery}
\label{alg:algorithm1}
  \KwSty{Data:}$\{x_i,t_i,u^t_i, u_i^{1-t_i}\}_{i=1}^N$\;
  \KwSty{Return:}$\{x_i,c_i,t_i,u^t_i, u_i^{1-t_i}\}_{i=1}^N$\; 
  \KwSty{Step 1:} Apply unsupervised clustering on $x_i$'s with $O(\sqrt{N})$ clusters\;
   \KwSty{Step 2:} Calculate average treatment effects per cluster group\;
   \KwSty{Step 3:} Apply unsupervised clustering with $k$ clusters on the cluster effects \;
   \KwSty{Step 4:} Assign in each $x_i$ the corresponding cluster $c_i$ as calculated in the previous step\;
 %  \caption{A general algorithm to automatically uncover hidden effect groups.}
\end{algorithm}

\subsection{Synthetic Data}
\label{sec:synthetic}
}
%%%%%%%%%%%%%%%%%%%%%%%%%%%

\textbf{X-learner \citep{kunzel2019metalearners}}, is a meta-learner that estimates heterogeneous treatment effects directly from \ite{}s. For the outcome and effect models of X-Learner we use a base gradient boosting learner with 400 estimators~\citep{friedman2001greedy} implemented in scikit-learn \citep{pedregosa2011scikit}. For the propensity model we use logistic regression.

{\bf Bayes Optimal} uses the \ite{}s to reconstruct the subpopulations, 
given the number of levels and the ground-truth outcome distribution \math{y(t,c)} from Figure~\ref{fig:data}. The Bayes optimal 
classifier is:
\math{c_{\text{Bayes}}=0} if \math{\ite\le 0.5},
\math{c_{\text{Bayes}}=1} if \math{0.5<\ite\le 1.5},
\math{c_{\text{Bayes}}=2} if \math{1.5<\ite}. 
We also use these thresholds to reconstruct 
subpopulations for X-learner's predicted \ite{}s. 

\noindent
    {\bf Note: Neither the bayes optimal thresholds nor the number of levels are available in practice.} So the subpopulations reconstructed  in the X-learner and
    Bayes Optimal benchmarks
    are optimistic because they have access to this forbidden information.
    Even still, we will see that our PCM method outperforms these optimistic
    benchmarks showcasing its power to extract subpopulations
    which outperform the competition, even when the competition has
    access to the forbidden information but PCM does not.

Let \math{c_i} be the level of subject \math{i} and \math{\itehat_i} the estimated \ite.
We define the error on subject \math{i} as
\math{|\mu_{c_i}-\itehat_i|}, and we report the mean absolute error
in the table below. 
Our algorithm predicts a level \math{\hat c_i} and uses its associated 
effect \math{\hat\mu_{\hat c_i}} as \math{\itehat_i}. The other methods predict 
\ite{} directly for which we compute mean absolute error.
\begin{center}{\renewcommand{\arraystretch}{1.2}
    \begin{tabular}[b]{c|c|cc|cc}    
    $n$ & \textbf{PCM} (this work)&\multicolumn{2}{c|}{\textbf{X-Learner}}&\multicolumn{2}{c}{\textbf{Bayes Optimal}}\\
    &&\color{red}Subpopulations&Predicted-\ite&\color{red}Subpopulations&Raw-\ite\\
    \hline
    \textbf{20K} & \textbf{0.35}$\pm$$\bm{0.39}$ & \color{red}3.04 $\pm$ 1.11& 3.07 $\pm$ 2.41&\color{red}4.57 $\pm$ 1.33 & 4.59 $\pm$ 3.49\\
    \textbf{200k} & \textbf{0.109}$\pm$$\bm{0.22}$& \color{red}1.44 $\pm$ 0.83& 1.50 $\pm$ 1.38&\color{red}4.22 $\pm$ 1.28 & 4.24 $\pm$ 3.22\\
    \textbf{2M} &\textbf{0.036}$\pm$$\bm{0.13}$ &\color{red} 0.34 $\pm$ 0.47& 0.46 $\pm$ 0.56& \color{red}4.01 $\pm$ 1.25& 4.03 $\pm$ 3.05\\
    \end{tabular}}
\end{center}
The entries in red use the forbidden information on the number of
levels and the thresholds for the levels to extract subpopulations
and predict \ite{}s from those optimally reconstructed
subpopulations. The red entries are not available in practice, but we
include them for comparison.
Our algorithm is about \math{10\times} better than existing benchmarks 
even though we do not use the forbidden information (number of levels and 
optimal thresholds). It is also clear that X-learner is significantly
better than Bayes optimal with just the raw \ite{}s. This is because
X-learner learns some form of internal smoothing. Our algorithm explicitly
does the smoothing in a provable way.

The next table shows subpopulation effects and red indicates 
the use of forbidden information (number of levels and optimal thresholds).
Ground truth effects are \math{\mu_0=0, \mu_1=1, \mu_2=2}.
Note: \math{\hat\mu_1} for X-learner and Bayes optimal are accurate,
an artefact of knowing the optimal thresholds (not realizable in practice).
\begin{center}   {\renewcommand{\arraystretch}{1.2}\tabcolsep8pt
\begin{tabular}{l|lll|lll|lll}
\math{n}& \multicolumn{3}{c|}{\textbf{PCM} (this work)}                               & \multicolumn{3}{c|}{\textbf{X-Learner}}                               & \multicolumn{3}{c}{\textbf{Bayes Optimal}}                               \\
& \textbf{$\hat\mu_0$} & \textbf{$\hat\mu_1$} & \textbf{$\hat\mu_2$} 
& \color{red}\textbf{$\hat\mu_0$} & \color{red}\textbf{$\hat\mu_1$} & \color{red}\textbf{$\hat\mu_2$} 
& \color{red}\textbf{$\hat\mu_0$} & \color{red}\textbf{$\hat\mu_1$} & \color{red}\textbf{$\hat\mu_2$} 
\\
                          \hline
\textit{\textbf{20K}}    
& -0.21                      & 0.91                     & 2.07                   
& \color{red}-2.5                     & \color{red}0.99                     & \color{red}4.44       
& \color{red}-3.94                      & \color{red}1.00                    & \color{red}5.99                     \\
\textit{\textbf{200K}} 
& 0.06                     & 0.963                     & 1.95                   
& \color{red}-1.16                      & \color{red}1.01                      & \color{red}2.87    
& \color{red}-3.62                    & \color{red}1.00                    & \color{red}5.61                    \\
\textit{\textbf{2M}} 
& 0.04                     & 0.996                      & 1.993          
& \color{red}-0.26                      & \color{red}0.99                      & \color{red}2.07  
& \color{red} -3.41 & \color{red} 1.00  & \color{red} 5.41
\end{tabular}}
\end{center}
A detailed comparison of our algorithm (PCM) with X-Learner and Bayes optimal
subpopulations is shown in Figure~\ref{tab:recHist}.
PCM clearly extracts the correct subpopulations. X-Learner and Bayes optimal,
even given the number of levels and optimal thresholds, does not come visually
close to PCM. Note, X-learner does display some structure but Bayes optimal
on just the \ite{}s is a disaster. This is further illustrated in the
\ite-histograms in the second row of Figure~\ref{tab:recHist}.
PCM clearly shows three effect-levels,
where as X-learner \ite{}s
and the raw \ite{}s suggest just one high variance effect-level.
The 3rd row in Figure~\ref{tab:recHist}
shows the confusion matrices for subpopulation assignment. 
The red indicates use of information forbidden in practice, however
we include it for comparison. 
The confusion matrix for PCM without
forbidden information clearly dominates the other methods which
use forbidden information. The high noise in the outcomes 
undermines the other methods, while PCM is robust.
In high noise settings, direct use of the \ite{}s without some form of pre-clustering fails.

\newcommand{\pKKKh}{./}
\newcommand{\pMMh}{./}
\newcommand{\pKKh}{./}

\newcommand{\adaptive}{ITE-ADAPTIVE-3CLUST}
\newcommand{\xl}{ITE-XML-3Clust}

\begin{figure}[!h]
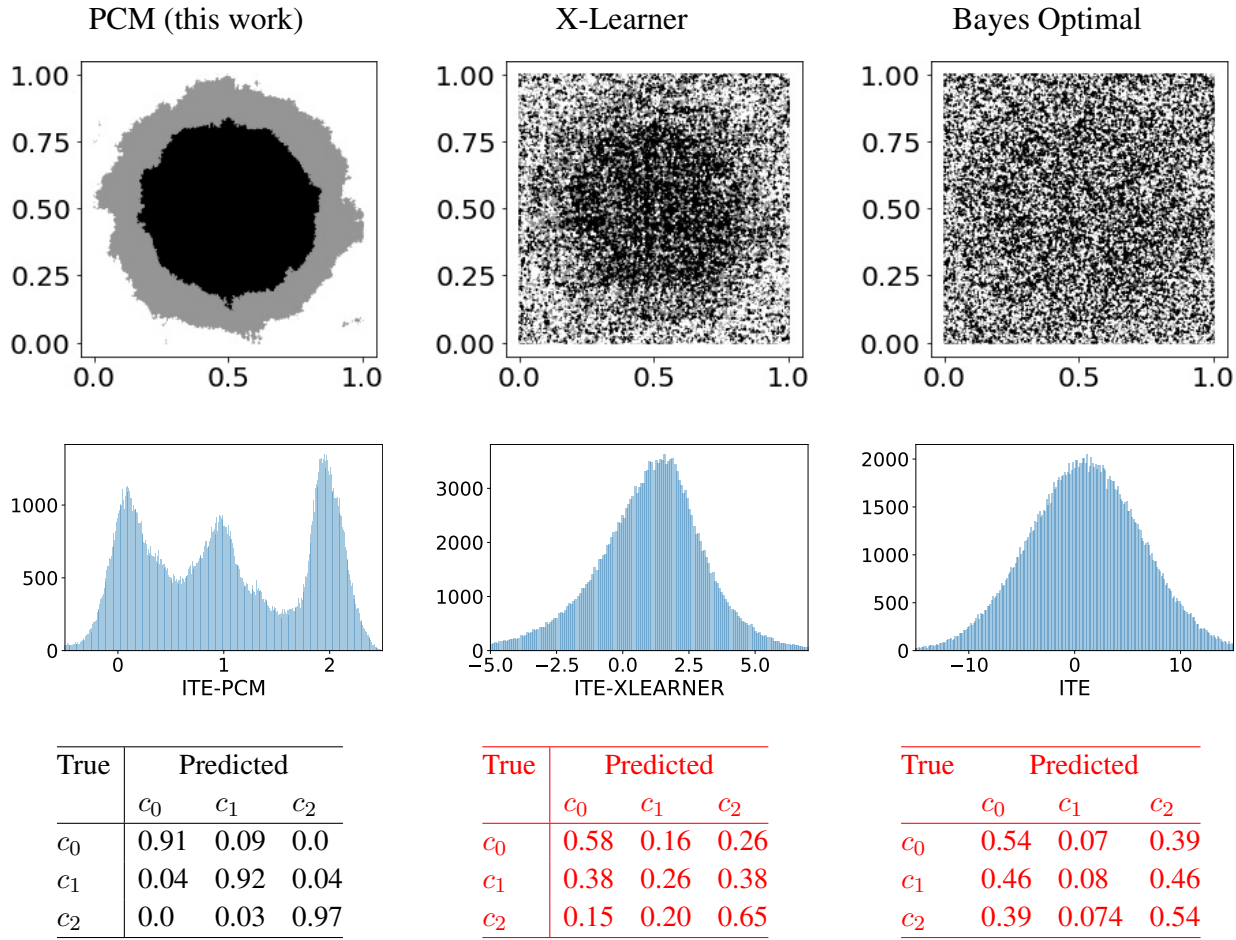

  \tabcolsep10pt\begin{tabular}{@{}ccc@{}}
    \large PCM (this work) & \large X-Learner & \large Bayes Optimal \\[8pt]
    \addpic{\pKK/\adaptive} &\addpic{\pKK/\xl} &\addpic{\pKK/ITE-2CGDBR}  \\[16pt]
    \addpic{\pKKh/adaptivehist} &\addpic{\pKKh/xlearnhist} &\addpic{\pKKh/ITE2hist} \\ \\
    
    \tabcolsep 5pt \renewcommand{\arraystretch}{1.05}
      \begin{tabular}{@{}l|lll@{}}
        \hline
        True&\multicolumn{3}{c}{Predicted}\\
        & \math{c_0} & \math{c_1} & \math{c_2} \\ \hline
        \math{c_0} & 0.91   & 0.09   & 0.0              \\
        \math{c_1} & 0.04  & 0.92    & 0.04            \\
        \math{c_2} & 0.0    & 0.03   & 0.97            \\ \hline
      \end{tabular}
    &
   \color{red}
   \tabcolsep 5pt \renewcommand{\arraystretch}{1.05}
   \begin{tabular}{@{}l|lll@{}}
     \hline
     True&\multicolumn{3}{c}{Predicted}\\
     & \math{c_0 } & \math{c_1 } & \math{c_2 } \\ \hline
     \math{c_0 } & 0.58    & 0.16     & 0.26             \\
\math{c_1 } & 0.38         & 0.26     & 0.38             \\
\math{c_2 } & 0.15         & 0.20     & 0.65             \\ \hline
\end{tabular}
    &
   \tabcolsep 5pt \renewcommand{\arraystretch}{1.05}
     \color{red}
     \begin{tabular}{@{}llll@{}}
       \hline
       True&\multicolumn{3}{c}{Predicted}\\
       & \math{c_0 } & \math{c_1 } & \math{c_2 } \\ \hline
       \math{c_0 } & 0.54   & 0.07    & 0.39             \\
\math{c_1 } & 0.46          & 0.08    & 0.46             \\
\math{c_2 } & 0.39          & 0.074   & 0.54             \\ \hline
\end{tabular}
    \end{tabular}

  \caption{
    Comparison of subpopulations from PCM, X-Learner and Bayes Optimal for
    200K points.
    {\bf Top.} PCM gives superior subpopulations
    without using the forbidden information used by X-learner and Bayes optimal (number of levels and optimal thresholds).
    {\bf Middle.} The \ite-histogram for PCM shows 3 distinct effects, while the other methods suggest a single high-variance effect. The \ite{}
    of a subject in PCM is the average (smoothed) \ite{} of the
    \math{\varepsilon}-hypercube centered on 
    the subject.
      {\bf Bottom.} Subpopulation confusion matrices show that PCM extracts the correct subpopulations. The other methods fail
      even with the forbidden information on the number
      of effect levels and the optimal effect-thresholds separating
      levels.
    %Our algorithm is able to extract the correct populations (histograms indicate three peaks). Contrarily, the competing methods indicate a unique effect population with high variance.
    }
    \label{tab:recHist}
  
\end{figure}
 
\paragraph{\bf Summary of experiments with synthetic data.}
The PCM
algorithm accurately extracts subpopulations at different effect-levels. 
Analysis of individual treatment effects fails when there is noise. Our experiments show that practice follows the theory (more detailed experiments, including how cluster homogeneity converges to 1, are shown in the appendix).
We note that there is a curse of dimensionality, namely the
convergence is at a rate \math{O(n^{-1/2d})}.

\section{Conclusion}
\label{sec:conclusion}
Our work amplifies the realm of 
causal analysis to non-targeted trials where
the treated population can consist of 
large subpopulations with different 
effects. 
Our algorithm uses a plug-and-play pre-cluster and merge strategy 
that provably untangles the different effects.
Experiments on synthetic data show a 
\math{10\times} or more improvement over existing HTE-benchmarks.

One main contribution is the idea to pre-cluster as a way of
smoothing individual \ite{}s. Even the Bayes Optimal classifier based
on individual \ite{}s (a classifier that is unavailable in practice)
breaks down when the individual \ite{}s are noisy -- multiple
heterogeneous subpopulations appear as one noisy subpopulation with
large variance in the effects.
This means that some form of the smoothing of the \ite{}s is needed
before subpopulation effects analysis. Existing methods based on
machine learning to learn the full effect function do indirectly
perform such smoothing but our PCM approach explicitly does so with
provable guarantees on consistency.
In our analysis, 
we did not attempt to optimize the rate of convergence.
Optimizing this rate could lead to improved algorithms.

Our work allows causal effects analysis to be used in settings
such as health interventions, where wide deployment over a mostly
healthy population results in a heterogeneous treated population. The
prevelance of healthy individuals would mask the effect on the sick population,
making methods such as ours useful.
Our methods can seemlessly untangle the effects without knowledge of
what sick and healthy mean. This is because our methods require only mild
technical assumptions which are likely to hold in practice, almost agnostic
to the causal analysis setting. 
This line of algorithms that focuses on extracting subpopulations
can also help in identifying inequities between
the subpopulations, by correlating patient features with the extracted
effect-subpopulations.

One significant technical
contribution is to reduce the untangling of subpopulation effects to 
a one dimensional clustering problem which we solve efficently.
This approach may be of independent interest beyond causal-effect analysis.
The effect is just a function that takes on \math{\ell} levels in our setting.
Our approach can be used to learn any function that takes on a finite
number of levels. It could also be used to learn a piecewise
approximation to an arbitrary continuous function on a compact set.

%\subsubsection*{Acknowledgments}
%The authors wish to thank Kristen Bennet and Jason Kuruzovich

\bibliographystyle{plainnat}
\bibliography{uai22bib.bib} 

\clearpage

\appendix
%\remove
{
\newpage
\clearpage
\appendix
\section*{Appendix}
\label{sec:appendix}
We provide more detailed experimental results, specifically results for different \math{n} (20K, 200K and 2M) and a comparison of different clustering methods in the pre-clustering phase: box-only, PCM (box plus 1 step of E-M improvement) and K-means. To calculate the counterfactual for treated subjects, we train a gradient boosted forest on the control population. 

\newcommand{\colTree}{0.68}
%%cluster figures define pathb, column width number colN
\newcommand{\pathb}{canary_agglo/big_font/clusters}
\newcommand{\colN}{0.4}
%\input{canary_app_figures_tex/surv-clusters}

%\newcolumntype{C}{>{\centering\arraybackslash}m{\wid}}
%\newcommand{\bfit}[1]{\textbf{\textit{#1}}}
%\newcommand{\wid}{0.6\columnwidth}
%\newcommand{\addpic}[1]{\includegraphics[width=\wid]{#1}}
%\newcommand{\pMM}{synthetic_figures/2M/reconstruction}
%\newcommand{\pKK}{synthetic_figures/200k/reconstruction}

%\input{synthetic_figures_tex/recJoin}

%\input{synthetic_figures_tex/histJoin}

\section{Convergence with \math{n}}

\subsection{Reconstructed Subpopulations}
We show subpopulation reconstructions for \math{n\in \{20K, 200K, 2M\}}. 
\begin{center}    \adjustbox{max width=\columnwidth}{%
    \begin{tabular}{@{}l*3{C}@{}}
   % \toprule
        &{\large PCM (this work)} & {\large X-Learner} & {\large Bayes Optimal} \\\\ %\midrule
                \bfit{\large 20k} &\addpic{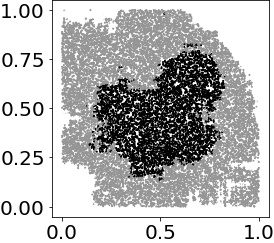} &\addpic{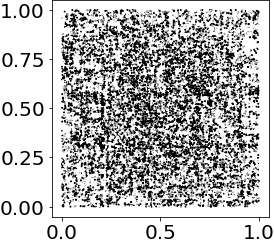} &\addpic{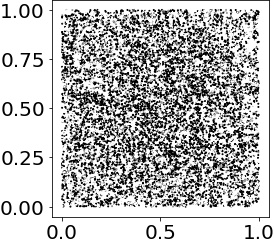} \\\\

        \bfit{\large 200k} &\addpic{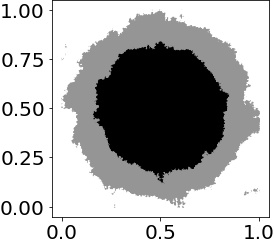} &\addpic{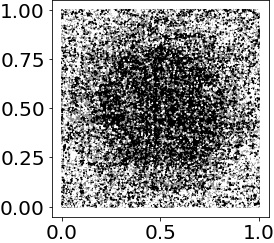} &\addpic{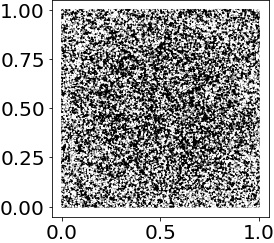} \\\\
        \bfit{\large 2M} &\addpic{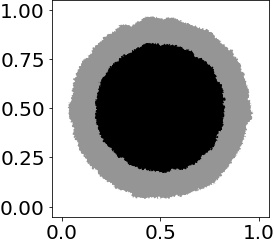} &\addpic{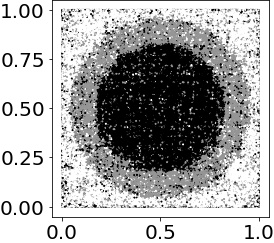} &\addpic{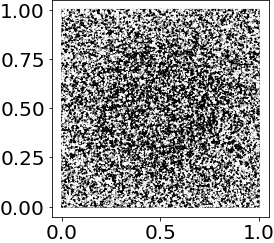}   
     %   \bottomrule
    \end{tabular}
    }
\end{center}
Even with just 20K points in this very noisy setting, PCM is able to extract some meaningful subpopulation structure, while none of the other methods can.

\newpage

\subsection{Predicted \ite{} Histograms for PCM, X-Learner and Raw \ite{}s}

We show the \ite{} histograms for \math{n\in \{20K, 200K, 2M\}}. 
\begin{center}
    \adjustbox{max width=\columnwidth}{%
    \begin{tabular}{@{}l*3{C}@{}}
   % \toprule
           &{\large PCM (our work)}  & {\large X-Learner} & {\large ITE} \\\\ %\midrule
                     \bfit{\large 20k} &\addpic{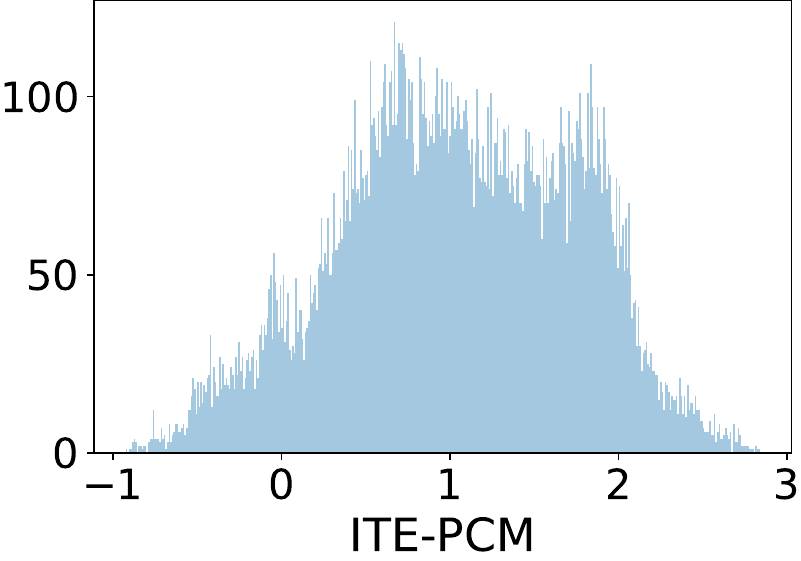} &\addpic{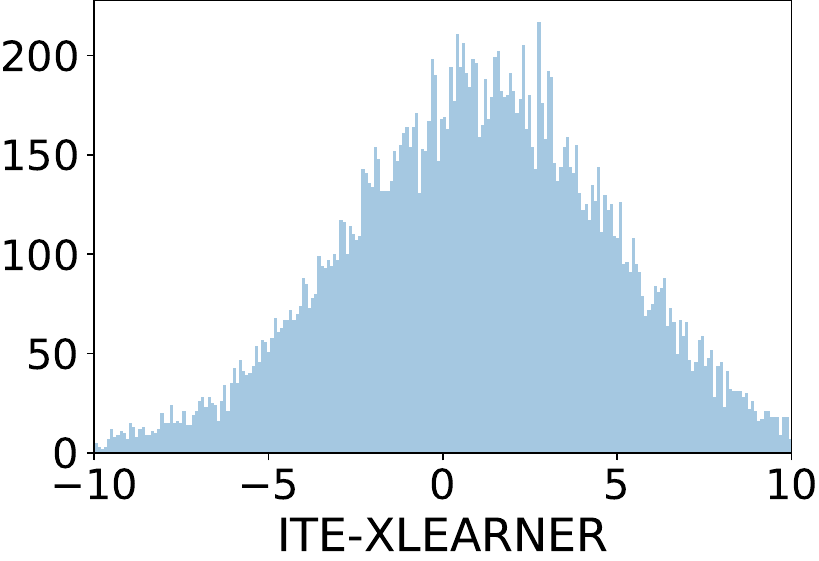} &\addpic{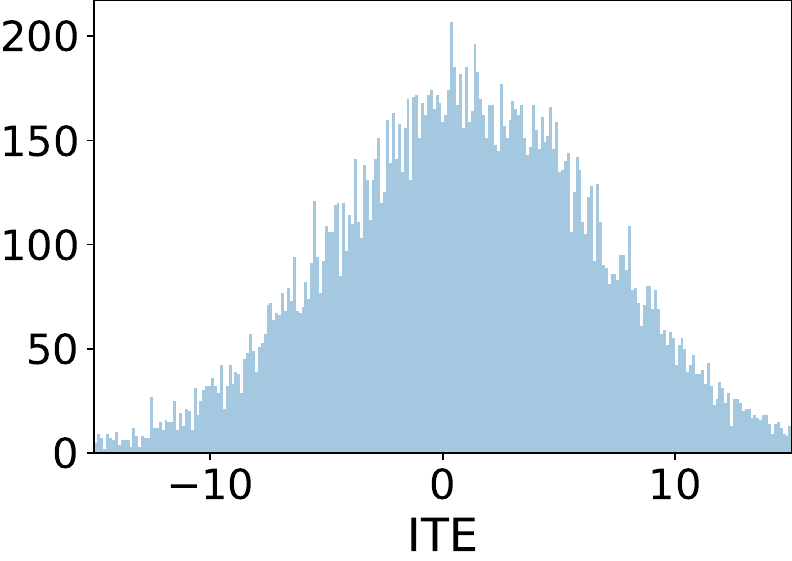} \\\\\\\\

          \bfit{\large 200k} &\addpic{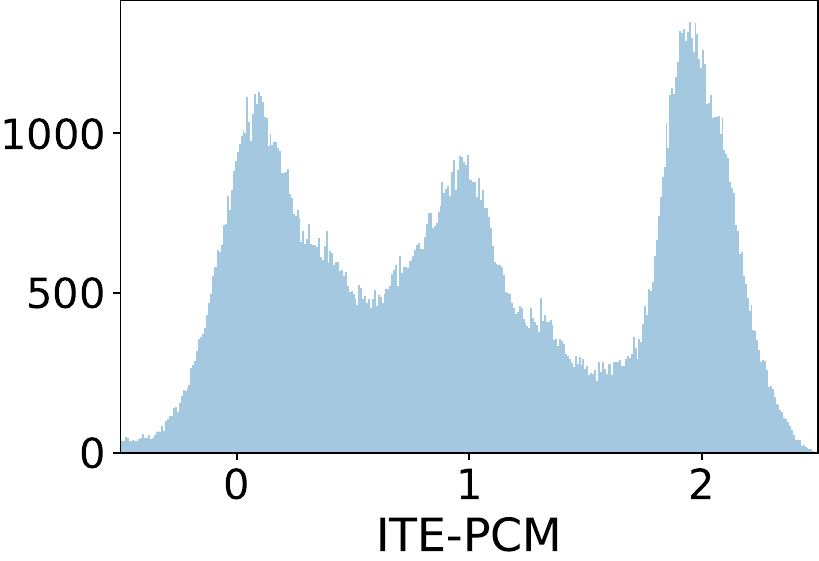} &\addpic{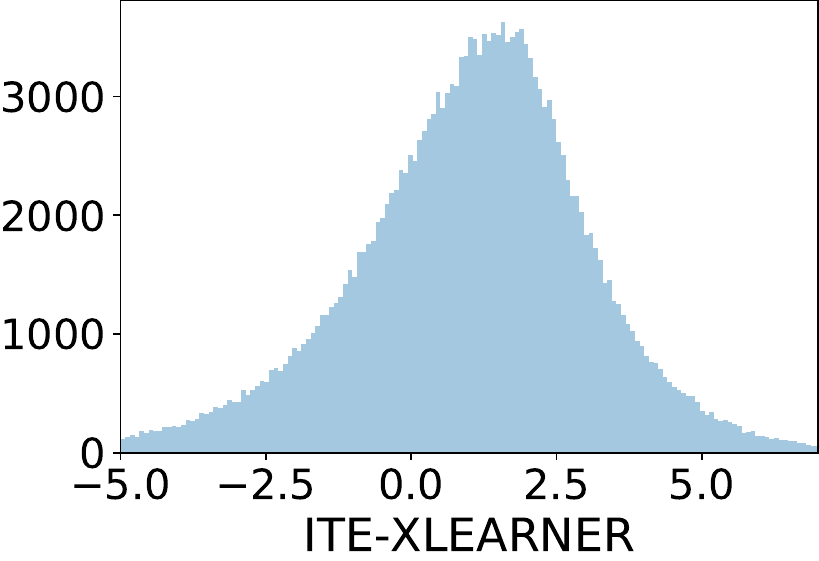} &\addpic{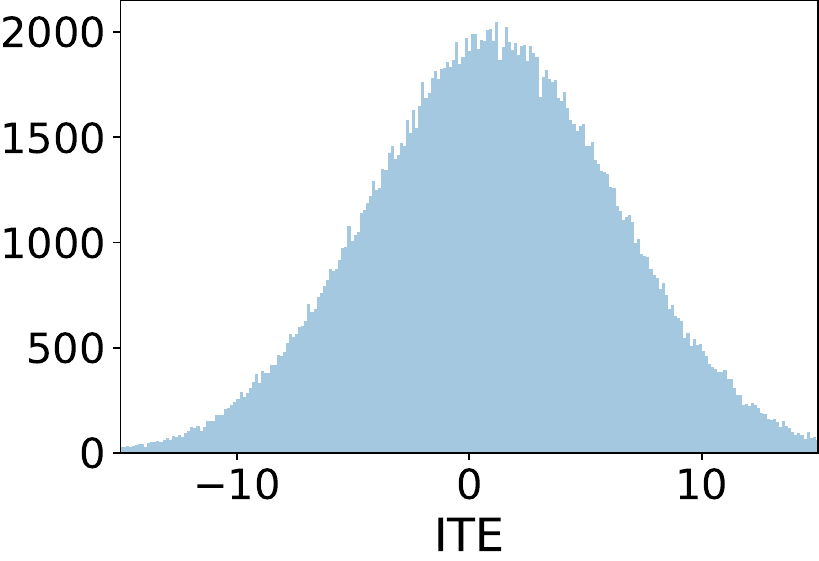} \\\\\\\\
        \bfit{\large 2M} &\addpic{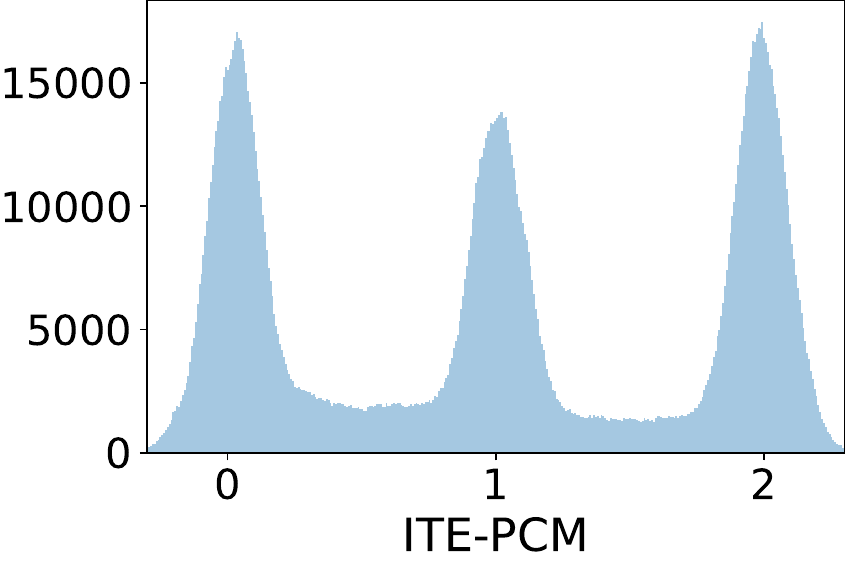} &\addpic{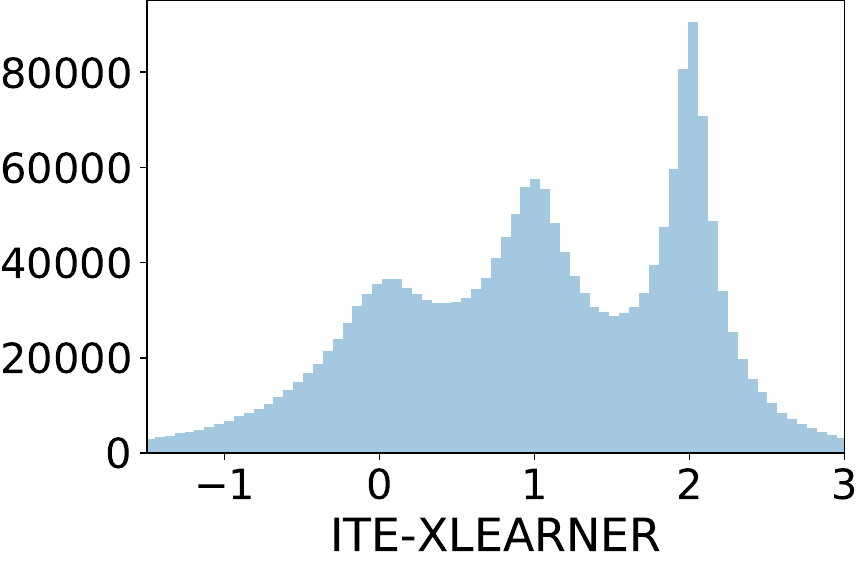} &\addpic{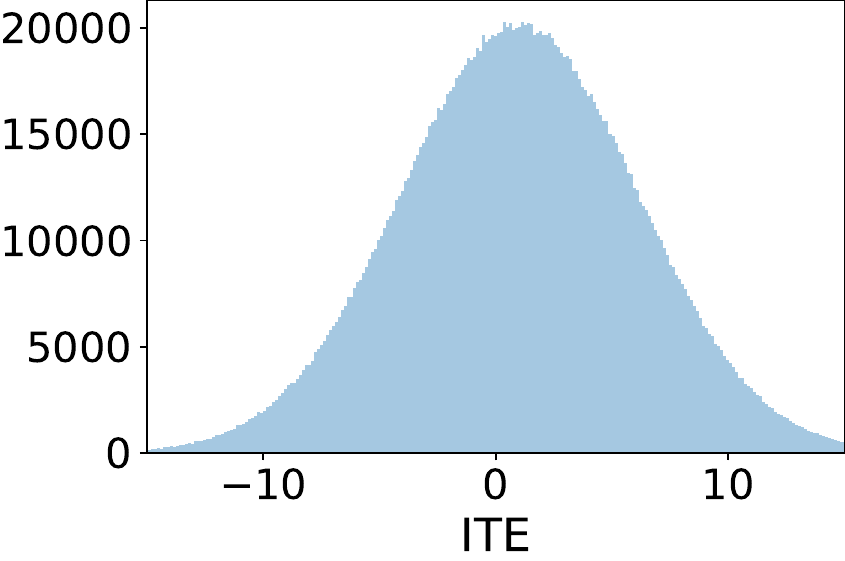} 
    %    \bottomrule
    \end{tabular}
    }
    \end{center}

\newpage
\section{Different Pre-Clustering Methods}

We show the reconstructed subpopulations and effect errors for different 
pre-clustering methods. Box-clustering without any E-M step is also provably consistent. Our algorithm PCM uses box-clustering followed by an E-M step to improve the subpopulations using smoothed \ite{}s as the average \ite{} in the \math{\varepsilon}-hypercube centered on the subject.
We also show K-means pre-clustering, for which we did not prove any theoretical
guarantees.
\bigskip

%SHOW Subpopulations for 20k,200k,2M and also the %error table.
\newcommand{\kmeans}{Kmeans2-Cl2}
\newcommand{\bx}{Box2-Cl2}
\textbf{Reconstruction of Effect Subpopulations.}
\begin{center}    \adjustbox{max width=\columnwidth}{%
    \begin{tabular}{@{}l*3{C}@{}}
   % \toprule
        &{\large PCM (this work)} & {\large BOX} & {\large KMEANS} \\\\ %\midrule
                \bfit{\large 20k} &\addpic{KKK-ITE-ADAPTIVE-3CLUST} &\addpic{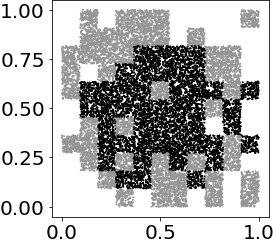} &\addpic{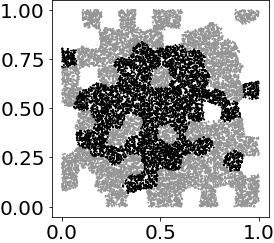} \\\\\\\\

        \bfit{\large 200k} &\addpic{KK-ITE-ADAPTIVE-3CLUST} &\addpic{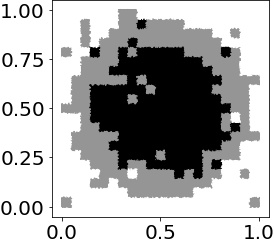} &\addpic{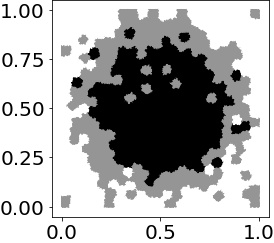} \\\\\\\\
        \bfit{\large 2M} &\addpic{MM-ITE-ADAPTIVE-3CLUST} &\addpic{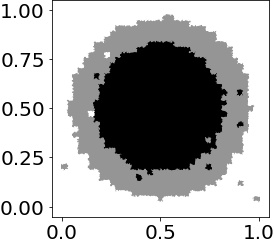} &\addpic{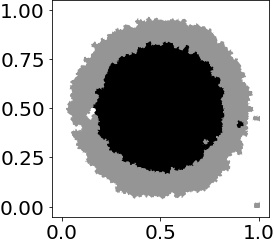}   
     %   \bottomrule
    \end{tabular}
    }
\end{center}
\clearpage

\textbf{Predicted \ite{} Histograms.}
\begin{center}
    \adjustbox{max width=\columnwidth}{%
    \begin{tabular}{@{}l*3{C}@{}}
   % \toprule
           &{\large PCM (our work)}  & {\large BOX} & {\large KMEANS} \\\\ %\midrule
                     \bfit{\large 20k} &\addpic{KKK-adaptivehist} &\addpic{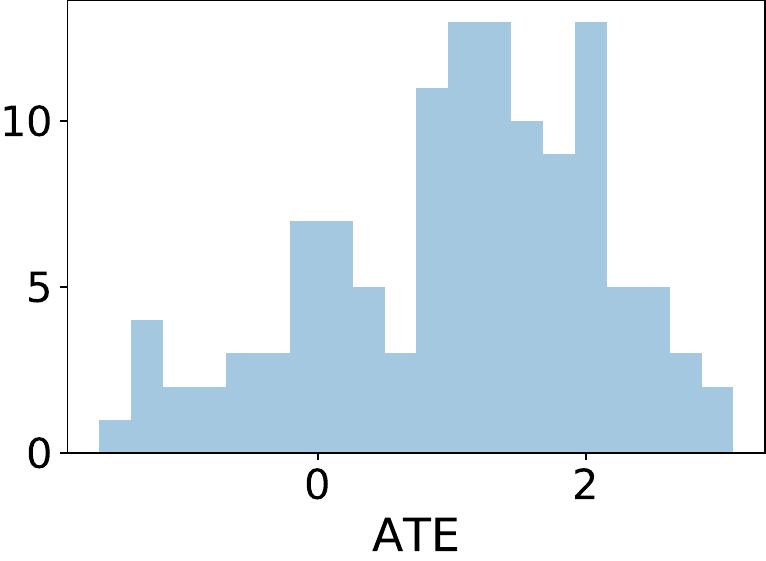} &\addpic{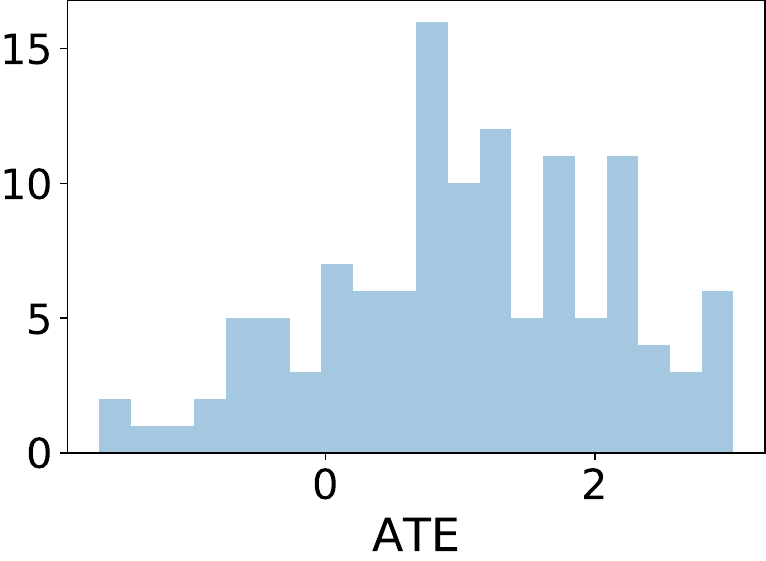} \\\\

          \bfit{\large 200k} &\addpic{KK-adaptivehist} &\addpic{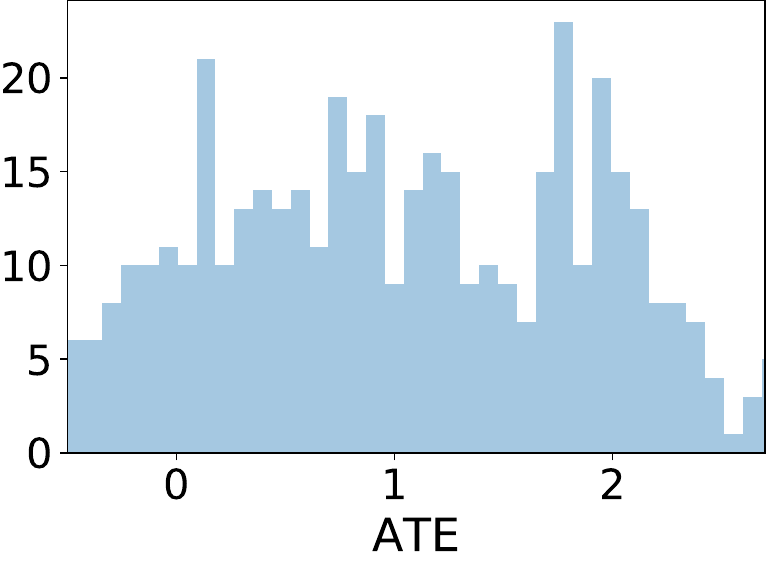} &\addpic{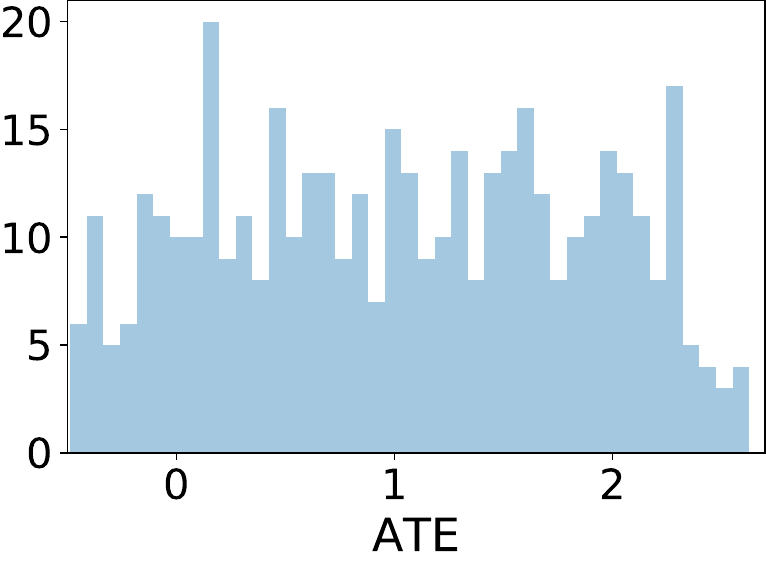} \\\\
        \bfit{\large 2M} &\addpic{MM-adaptivehist} &\addpic{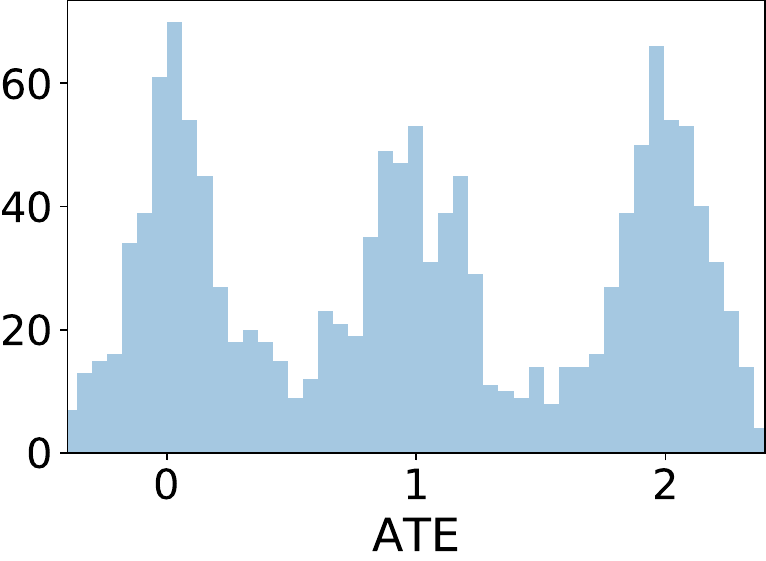} &\addpic{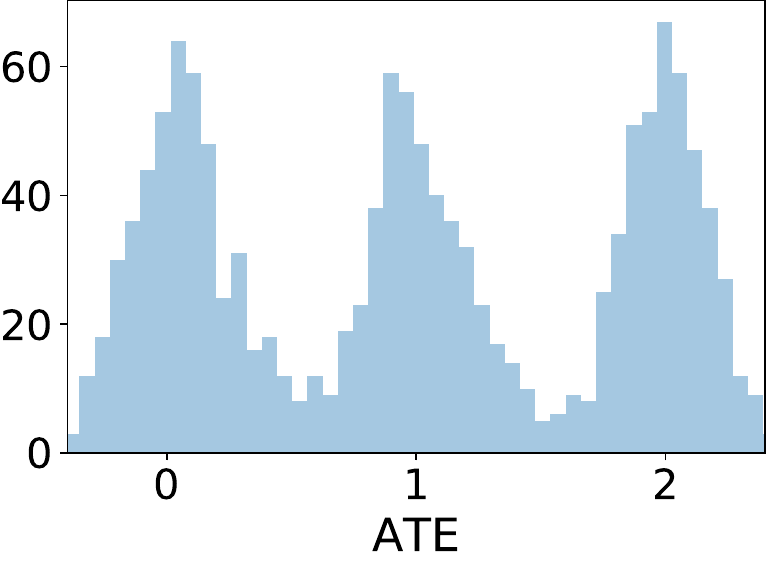} 
    %    \bottomrule
    \end{tabular}
    }
    \end{center}
\textbf{Error Table.}
\begin{center}{\renewcommand{\arraystretch}{1.5}
    \begin{tabular}[b]{c|c|c|c}    
    $n$ & \textbf{PCM} (this work)&\textbf{BOX}&\textbf{KMEANS}\\
   % &&\color{red}Subpopulations&Predicted-\ite&\color{red}Subpopulations&Raw-\ite\\
    \hline
    \textbf{20K} & \textbf{0.35}$\pm$$\bm{0.39}$    & 0.50 $\pm$ 0.52   & 0.54 $\pm$ 0.50\\
    \textbf{200k} & \textbf{0.109}$\pm$$\bm{0.22}$  & 0.17 $\pm$ 0.35  & 0.20 $\pm$ 0.37\\
    \textbf{2M} &\textbf{0.036}$\pm$$\bm{0.13}$     & 0.078 $\pm$ 0.214   & 0.065 $\pm$ 0.20\\
    \end{tabular}}
\end{center}
\newpage
\section{Cluster Homogeneity}
To further show how practice reflects the theory, we plot average cluster homogeneity versus \math{n}. The cluster homogeneity is the fraction of points in a cluster that are from its majority level. 
Our entire methodology relies on the pre-clustering step producing a vast majority of homogeneous clusters.
The rapid convergence to homogeneous clusters enables us to identify the correct subpopulations and the corresponding effects via pre-cluster and merge.
\newcommand{\pathHom}{std5PCM.pdf}
\begin{center}
     \includegraphics[width = 0.5\columnwidth]{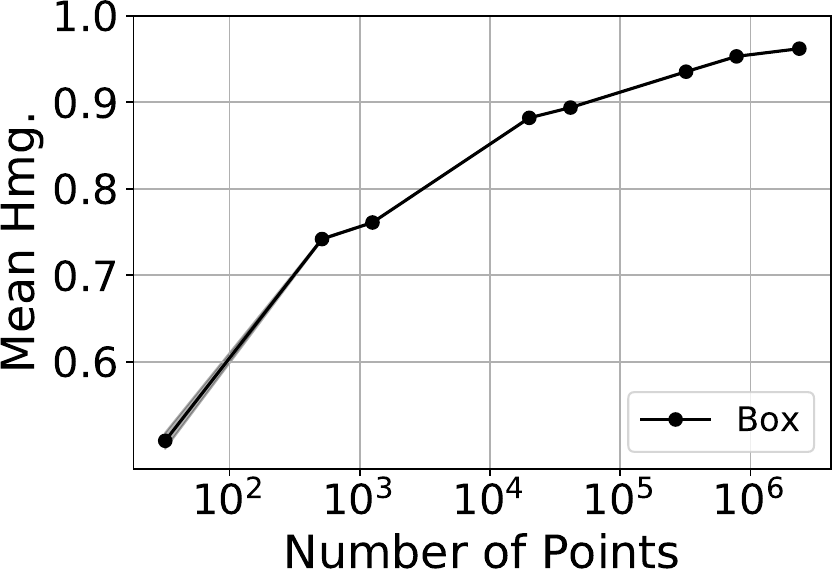}
\end{center}

}

\end{document}